\documentclass{article}
\usepackage{ijcai22}

\usepackage{times}

\usepackage{soul}
\usepackage{url}
\usepackage[hidelinks]{hyperref}
\usepackage[utf8]{inputenc}
\usepackage[small]{caption}
\usepackage{graphicx}
\usepackage{amsmath}
\usepackage{booktabs}
\title{Learning Mixtures of Random Utility Models with Features from Incomplete Preferences}

\usepackage{natbib}  

\usepackage{booktabs}       
\usepackage{amsfonts}       
\usepackage{microtype}      
\usepackage{amssymb,amsmath,amsthm}
\usepackage{algorithm}
\usepackage{algorithmic}
\usepackage{color}
\usepackage{bbm}
\usepackage[switch]{lineno}

\newtheorem{thm}{Theorem}
\newtheorem{dfn}{Definition}

\newtheorem{lem}{Lemma}
\newtheorem{ex}{Example}
\newtheorem{ass}{Assumption}

\newtheorem{coro}{Corollary}

\newcommand{\rank}{\text{rank}}
\newcommand{\norm}{\text{norm}}
\newcommand{\bignorm}{\text{Norm}}
\newcommand{\ma}{\mathcal A}
\newcommand{\mi}{\mathcal I}
\newcommand{\ml}{\mathcal L}
\newcommand{\mt}{\mathcal T}
\newcommand{\mm}{\mathcal M}
\newcommand{\mx}{\mathcal X}
\newcommand{\xji}{\vec x_{ji}}
\newcommand{\xjp}[1]{\vec x_{j#1}}
\newcommand{\xjir}{x_{ji, r}}
\newcommand{\plx}{\text{PL}_\mx}
\newcommand{\rumx}{\text{RUM}_\mx}
\newcommand{\rumxp}{\text{RUM}_{\mx, p}}
\newcommand{\plxp}{\text{PL}_\mx\text{-TO}}

\newcommand{\plyz}{\text{PL}_{Y, Z}}
\newcommand{\plyzp}{\text{PL}_{Y, Z}\text{-TO}}
\newcommand{\cll}{\text{CLL}_\mm}

\newcommand\cl{\text{CL}_\mm}

\newcommand{\appLem}[2]{\vspace{1mm}\noindent{\bf Lemma~\ref{#1}.} {\em #2}}
\newcommand{\appThm}[2]{\vspace{1mm}\noindent{\bf Theorem~\ref{#1}.} {\em #2 \vspace{1mm}}}
\newcommand{\appCoro}[2]{\vspace{1mm}\noindent{\bf Corollary~\ref{#1}.}{\em #2 \vspace{1mm}}}

\author{
Zhibing Zhao$^1$\footnote{Work done while at Rensselaer Polytechnic Institute.}\and
Ao Liu$^2$\and
Lirong Xia$^2$
\affiliations
$^1$Microsoft, 555 110TH Ave NE, Bellevue, WA, 98004\\
$^2$Rensselaer Polytechnic Institute, 110 8th Street, Troy, NY, 12180
\emails
zhaozb08@gmail.com,
liua6@rpi.edu,
xial@cs.rpi.edu
}

\begin{document}

\maketitle

\begin{abstract}
Random Utility Models (RUMs), which subsume Plackett-Luce model (PL) as a special case, are among the most popular models for preference learning. In this paper, we consider RUMs with features and their mixtures, where each alternative has a vector of features, possibly different across agents. Such models significantly generalize the standard PL and RUMs, but are not as well investigated in the literature. We extend mixtures of RUMs with features to models that generate incomplete preferences and characterize their identifiability. For PL, we prove that when PL with features is identifiable, its MLE is consistent with a strictly concave objective function under mild assumptions, by characterizing a bound on root-mean-square-error (RMSE), which naturally leads to a sample complexity bound. We also characterize identifiability of more general RUMs with features and propose a generalized RBCML to learn them. Our experiments on synthetic data demonstrate the effectiveness of MLE on PL with features with tradeoffs between statistical efficiency and computational efficiency. Our experiments on real-world data show the prediction power of PL with features and its mixtures.
\end{abstract}

\section{Introduction}
\label{sec:intro}

Preference learning is a fundamental machine learning problem in a wide range of applications such as discrete choice analysis~\citep{McFadden73:Conditional}, marketing~\citep{Berry95:Automobile}, meta-search engines~\citep{Dwork01:Rank}, information retrieval~\citep{Liu09:Learning}, recommender systems~\citep{Baltrunas10:Group}, crowdsourcing~\citep{Chen13:Pairwise,Mao13:Better},  social choice~\citep{Azari12:Random}, among many others. Plackett-Luce model (PL)~\citep{Plackett75:Analysis,Luce59:Individual} is one of the most popular statistical models for preference learning due to its interpretability and computational tractability. In the standard PL, each alternative is parameterized by a real number $\theta$ such that $e^\theta$ represents the alternative's ``quality''. The higher the quality is, the more likely the alternative is ranked higher by an agent. 

Extensions of the standard PL has been proposed and studied mostly in three dimensions. The first dimension is {\em PL with features}, where features of the agents and/or the alternatives are given, and the model is parameterized by the relationship (often linear, see Definition~\ref{def:fpl}) between the features and the quality of the alternatives. Examples include the conditional logit model~\citep{McFadden73:Conditional}, the BLP model~\citep{Berry95:Automobile}, and bilinear models~\citep{Azari13:Preference,Schafer18:Dyad,Zhao18:Cost}.

\begin{table*}[htp]
    \centering
    \begin{tabular}{lccccc}
    \toprule
     & features & top-$l$ & mixtures & identifiability & consistency\\
    \hline
  \bf This work & \checkmark & \checkmark & \checkmark & \checkmark (Thm.~\ref{thm:idplx}, \ref{thm:idkplx}, Coro.~\ref{coro:idbilinear}, \ref{coro:idkplx})& \checkmark (Thm.~\ref{thm:msebound})\\
\hline
  \citep{Tkachenko16:Plackett} & \checkmark & \checkmark & \checkmark & &\\
\hline
\citep{Yildiz20:Fast} & \checkmark & \checkmark & & &\\
\hline
    \begin{tabular}{@{}l}\citep{McFadden73:Conditional},\\\citep{Berry95:Automobile}\end{tabular} & \checkmark & top-$1$ & & \checkmark & \checkmark\\
\hline
    \citep{Grun08:Identifiability} & \checkmark & top-$1$ & & \checkmark &\\
\hline
    \citep{Schafer18:Dyad} & \checkmark & linear orders  & & \checkmark &\\
\hline
    \begin{tabular}{@{}l}\citep{Zhao16:Learning},\\ \citep{Zhao18:Learning},\\\citep{Zhao19:Learning}\end{tabular} & & linear orders & \checkmark & \checkmark & \checkmark\\
\hline
    \citep{Chierichetti18:Learning} & & \checkmark & \checkmark & \checkmark & \\
    \bottomrule
    \end{tabular}
    \caption{The model, identifiability, and consistency of this work compared with previous work.}
    \label{tab:related}
\end{table*}

The second dimension is {\em PL for partial preferences}, where the data consist of partial preferences, often represented by partial orders over the alternatives. Due to the hardness of tackling general partial orders~\citep{Liu19:Learning}, most previous work focused on natural sub-cases, including choice data~\citep{Train09:Discrete}, top-ranked orders (top-$l$)~\citep{Mollica17:Bayesian}, and pairwise preferences~\citep{Hullermeier08:Label}. 
In particular, in a top-$l$ order, the agent reports a linear order over her most-preferred $l$  alternatives. Top-$l$ orders generalize standard PL ($l=m-1$, where $m$ is the number of alternatives) and choice data ($l=1$). 

The third dimension is {\em PL mixture models}, where $k\ge 1$ PL models are combined via a {\em mixing coefficients $\vec\alpha=(\alpha_1,\ldots,\alpha_k)$} and each $\alpha_i$ represents the probability that the data is generated from the $i$-th PL. Mixtures of PLs  provide better fitness to data~\citep{Tkachenko16:Plackett} and are a popular method for clustering~\citep{Gormley08:Exploring}, but they are generally hard to compute and are prone to criticisms on interpretability and trustworthiness due to their (non-)identifiability~\citep{Zhao16:Learning}. 

While there is a large literature on standard PL and its extensions in each of the three dimensions, little is known about the generalization of PL in all three dimensions simultaneously, i.e.~mixtures of PL models with features for partial preferences. The literature on general RUMs and their extensions is far more limited. 
The problem is already highly challenging for top-$l$ orders---to the best of our knowledge, only one previous work studied mixtures of PL models with features for top-$l$ orders~\citep{Tkachenko16:Plackett}, where an EM algorithm was proposed yet no theoretical guarantees on the model or the algorithm were given. 

Motivated by the lack of theoretical understandings of the general PL extensions and learning algorithms, we ask the following question. {\bf When and how can preference learning be done for (mixtures of) RUMs with features for incomplete preferences?}

The question is highly challenging and the answer is not fully known even for many well-studied sub-cases, such as (non-mixture) PL with features for top-$l$ orders and mixtures of standard PL (without features) for linear orders. In this paper, we provide the first answers to the  question for PL with features for top-$l$ orders by characterizing their {\em identifiability}, {\em consistency}, and {\em sample complexity} of their MLEs. We also provide the first generic identifiability result for its mixture models as well as more general RUMs with features.

Identifiability is a fundamental property of statistical models that is important to explainability and trustworthiness of decisions, which are particularly relevant in preference learning scenarios~\citep{Gormley08:Exploring}. Identifiability requires that different parameters of the model lead to different distributions over data. If a model is non-identifiable, then sometimes the explanations of the learned parameters and corresponding decisions can be provably wrong, because there may exist another parameter that fits the data equally well, yet whose explanation and corresponding decisions are completely different. See Example~\ref{ex:running} for an illustrative example. Additionally, the identifiability of a model is necessary for any algorithm to be consistent or have finite sample complexity.

\noindent{\bf Our Contributions.} In this paper, we provide the first theoretical characterizations of identifiability for mixtures of $k\ge 1$ PLs with features for top-$l$ orders, denoted by $k$-$\plxp$, where $\mx$ is the feature matrix. We note that in $k$-$\plxp$, each agent submits a top-$l$ order for possibly different $l$. 

\noindent{\bf Identifiability.} We provide necessary and sufficient conditions for $\plxp$ and its special case called the bilinear model, denoted by $\plyzp$, to be identifiable in Theorem~\ref{thm:idplx} and {Corollary~\ref{coro:idbilinear}}, respectively. Even though $k$-$\plxp$ is not identifiable for any $k\ge 2$, we provide a sufficient condition in Theorem~\ref{thm:idkplx} for a parameter $k$-$\plxp$ to be identifiable, which leads to the {\em generically identifiability} of $2$-$\plxp$ in Corollary~\ref{coro:idkplx}.  
It suggests that identifiability may not be a practical concern if the condition is satisfied. We also characterize identifiability of RUMs with features in the appendix.

\noindent{\bf Strict concavity, consistency, and sample complexity of MLE.}  We provide a necessary and sufficient condition for the MLE of $\plxp$ to be strictly concave in Theorem~\ref{thm:log-concavity} and bound on the RMSE of MLE of $\plxp$ in Theorem~\ref{thm:msebound}, which implies the consistency of MLE and a sample complexity bound.
 
Our experiments on synthetic data demonstrate the performance of MLE and the tradeoffs between statistical efficiency and computational efficiency when learning from different top-$l$ preferences for $\plxp$. For $k$-$\plxp$, we propose an EM algorithm and show the prediction power of different configurations of $k$-$\plxp$.

{\noindent\bf Related Work and Discussions.} Table~\ref{tab:related} summarizes related works on PL and its extensions that are close to ours. As discussed above, there is a large literature on PL and its extensions in each of the three dimensions, yet no theoretical result is known even for the special non-mixture case $\plxp$ for top-$l$ orders in general. \cite{Tkachenko16:Plackett} is the only previous work we are aware of that tackles $k$-$\plxp$, which does not provide theoretical guarantees. Even for $\plxp$, we are only aware of another recent paper~\citep{Yildiz20:Fast}, which proposed an ADMM-based algorithm for computing the MLE, but it is unclear whether their algorithm converges to the ground truth because the consistency of MLE was unknown, which is a direct corollary of Theorem~\ref{thm:msebound}. To the best of our knowledge, our identifiability results (Theorems~\ref{thm:idplx}, \ref{thm:idkplx}, Corollaries~\ref{coro:idbilinear}, \ref{coro:idkplx}) and the RMSE bound (Theorem~\ref{thm:msebound}) are the first for ($k$-)$\plxp$ even for linear orders ($l=m-1$).

Mixtures of PLs with features for choice data (top-$1$) are well studied and can be dated back to the classical conditional logit model~\citep{McFadden73:Conditional} and the BLP model~\citep{Berry95:Automobile}. PL with bilinear features, which is a special case of $\plxp$, has been studied in the literature~\citep{Azari13:Preference,Schafer18:Dyad,Zhao18:Cost}. There is a large literature on  standard PL (without features) and its mixture models~\citep{Hunter04:MM,Soufiani13:Generalized,Negahban17:Rank,Maystre15:Fast,Khetan16:Data,Zhao18:Composite,Gormley08:Exploring,Oh14:Learning,Zhao16:Learning,Chierichetti18:Learning,Zhao19:Learning,Liu19:Learning}. Our setting is more general. 

Recently, \citet[Proposition 1]{Schafer18:Dyad} provided a sufficient condition for (non-mixture) PL with bilinear features to be identifiable. However, their result is flawed due to the missing conditions on the ranks of feature matrices. See Example~\ref{ex:bilinear} for more details. Our Corollary~\ref{coro:idbilinear} provides a necessary and sufficient condition for the identifiability of more general models. 
\citet{Zhao16:Learning} characterized the conditions on $k$ (the number of components in the mixture model) and $m$ (the number of alternatives) for mixtures of standard PLs to be (non-)identifiable. \citet{Zhao19:Learning} characterized (non-)identifiability of mixtures of PLs (without features) for structured partial orders. \citet{Grun08:Identifiability} characterized conditions for mixtures of multinomial logit models to be identifiable. These results do not subsume our results because the model studied in this paper is more general.

Consistency of MLE was proven for the conditional logit model~\citep{McFadden73:Conditional} assuming that each agent provides multiple choice data. Or equivalently, agents with the same features are repeatedly observed. This assumption may not hold in the big data era, where the feature space can be extremely large and it is unlikely that agents would have identical features. Our proof of the RMSE bound on MLE, which implies consistency, tackles exactly this case and is inspired by the proof of~\citep[Theorem 8]{Khetan16:Data} for standard (non-mixture)  PL. Unlike the standard PL where the Hessian matrix is negative semidefinite with at least one zero eigenvalue~\citep{Khetan16:Data}, the Hessian matrix of the model studied in this paper does not have zero eigenvalues. 
Due to this difference, the techniques in proving the sample complexity bound in~\citep{Khetan16:Data} cannot be directly extended to our setting. 

Due to the space constraint, we focus on Plackett-Luce model in the main paper, while defer results on RUMs to the appendix.

\section{Preliminaries}

Let $\ma=\{a_1, \ldots, a_m\}$ denote the set of $m$ alternatives. Let $\{1, \ldots, n\}$ denote the set of $n$ agents. Given an agent $j$, each alternative is characterized by a column vector of $d\ge 1$ features $\xji\in\mathbb R^d$.  
For any $r=1, \ldots, d$, let $\xjir$ denote the $r$th feature of $\xji$. A linear order, which is a transitive, antisymmetric, and total binary relation, is denoted by $R=a_{i_1}\succ\ldots\succ a_{i_m}$, where $a_{i_1}\succ a_{i_2}$ means that the agent prefers $a_{i_1}$ over $a_{i_2}$. Let $\ml(\ma)$ denote the set of all linear orders. A ranked top-$l$ (top-$l$ for short) order has the form $O=a_{i_1}\succ a_{i_2}\ldots\succ a_{i_l}\succ\text{others}$. It is easy to see that a linear order is a special top-$l$ order with $l=m-1$. Let $\mt(\ma)$ denote the set of all top-$l$ orders for all $l\in\{1, 2, \ldots, m-1\}$. An $l$-way order has the form $O=a_{i_1}\succ a_{i_2}\ldots\succ a_{i_l}$. Let $\mi(\ma)$ denote the set of all $l$-way orders for all $l\in\{2, \ldots, m\}$. 

\begin{dfn}[Plackett-Luce model (PL)] 
The parameter space is $\Theta=\mathbb R^m$. The sample space is $\ml(\ma)^n$. Given a parameter $\vec\theta\in\Theta$, the probability of any linear order $R=a_{i_1}\succ a_{i_2}\succ\ldots\succ a_{i_m}$ is 
$
\Pr\nolimits_{\text{PL}}(R|\vec\theta)=\prod^{m-1}_{p=1}\frac {\exp(\theta_{i_p})} {\sum^m_{q=p}\exp(\theta_{i_q})}.$
\end{dfn}

It follows that the marginal probability for any top-$l$ order $R=a_{i_1}\succ a_{i_2}\ldots\succ a_{i_l}\succ\text{others}$ is $\Pr\nolimits_{\text{PL}}(R|\vec\theta)=\prod^{l}_{p=1}\frac {\exp(\theta_{i_p})} {\sum^m_{q=p}\exp(\theta_{i_q})}.$
In the literature, a normalization constraint on $\vec\theta$, e.g.~ $\sum_i\theta_i=0$, is often required to make the model identifiable. In this paper we do put such a constraint since it is more convenient to extend the current definition to PL with features.

Let $X_j=[\vec x_{j1}, \ldots, \vec x_{jm}]\in\mathbb R^{d\times m}$ denote the feature matrix for agent $j$, and let $\mx=[X_1, \ldots, X_n]\in\mathbb R^{d\times mn}$ denote the feature matrix that concatenates the features for all agents. The Plackett-Luce model with features is defined as follows.

\begin{dfn}[Plackett-Luce model with features ($\plx$)]\label{def:fpl} Let $\mx\in\mathbb R^{d\times mn}$ denote a feature matrix. 
The parameter space is 
$\Theta={\mathbb R}^d$. The sample space is $\ml(\ma)^n$. For any parameter $\vec\beta\in \Theta$, the probability of any linear order $R_j=a_{i_1}\succ a_{i_2}\succ\ldots\succ a_{i_m}$ given by agent $j$ is 
$
\Pr\nolimits_{\plx}(R_j|\vec\beta)=\prod^{m-1}_{p=1}\frac {\exp(\vec\beta\cdot\xjp{i_p})} {\sum^m_{q=p}\exp(\vec\beta\cdot\xjp{i_q})}$.
\end{dfn}

We note that {\em all feature matrices are assumed given, i.e., not part of the parameter of any model in this paper}. PL with bilinear features~\citep{Azari13:Preference} is a special case of $\plx$, where each agent $j\in\{1, \ldots, n\}$ is characterized by a column feature vector $\vec y_j\in {\mathbb R}^L$ and each alternative $a_i\in\ma$ is characterized by a column feature vector $\vec z_i\in {\mathbb R}^K$. We note that for any $i\in\{1, \ldots, m\}$, $\vec z_i$ is the same across all agents. Let $Y=[\vec y_1, \ldots, \vec y_n]\in\mathbb R^{L\times n}$ denote the agent feature matrix and $Z=[\vec z_1, \ldots, \vec z_m]\in\mathbb R^{K\times m}$ denote the alternative feature matrix. PL with bilinear features is defined as follows.

\begin{dfn}[Plackett-Luce model with bilinear features $\plyz$]\label{def:bfpl}  Let $Y\in\mathbb R^{L\times n}$ denote an agent feature matrix and let $Z\in\mathbb R^{K\times m}$ denote an alternative feature matrix.  
The parameter space consists of matrices $\Theta=\mathbb R^{K\times L}$. The sample space is $\ml(\ma)^n$. Given a parameter $B\in \Theta$, the probability of any linear order $R_j=a_{i_1}\succ a_{i_2}\succ\ldots\succ a_{i_m}$ given by agent $j$ is 
$
\Pr\nolimits_{\plyz}(R_j|B)=\prod^{m-1}_{p=1}\frac {\exp(\vec z^\top_{i_p} B\vec y_j)} {\sum^m_{q=p}\exp(\vec z^\top_{i_q} B\vec y_j)}.$
\end{dfn}

$\plyz$ can be viewed as a special case of $\plx$ by letting $\mx=Y\otimes Z$ and vectorizing $B$ accordingly. Before defining mixtures of PL with features, we recall the definition of mixtures of PL as follows.

\begin{dfn}[$k$-PL]\label{def:mpl}
For any  $k\geq 1$, the mixture of $k$ Plackett-Luce models is defined as follows. The sample space is $\mathcal L(\ma)^n$. The parameter space has two parts. The first part is the {\em mixing coefficients} $\vec\alpha = \left(\alpha_1,\cdots,\alpha_k\right)$ with $\vec \alpha\ge 0$ and $\vec\alpha\cdot\vec 1 =1$. The second part is $\left(\vec\theta^{(1)},\cdots,\vec\theta^{(k)}\right)$, where $\vec\theta^{(r)}\in\Theta$ is the parameter of the $r$-th PL component. The probability of any linear order $R$ is
$\Pr\nolimits_{k\text{-PL}}(R|\vec\theta) = \sum^k_{r=1}\alpha_r\Pr\nolimits_{\text{PL}}(R|\vec\theta^{(r)}).$
\end{dfn}

The mixture of $k$ Plackett-Luce models with features is therefore defined as follows.

\begin{dfn}[Mixtures of $k$ Plackett-Luce models with features ($k$-$\plx$)] Let $\mx\in\mathbb R^{d\times mn}$ denote the feature matrix. 
The parameter space $\Theta$ has two parts. The first part is the vector of mixing coefficients $\vec\alpha=(\alpha_1, \alpha_2, \ldots, \alpha_k)$ and the second part is $(\vec\beta^{(1)}, \vec\beta^{(2)}, \ldots, \vec\beta^{(k)})$, where for $r=1, \ldots, k$, $\vec\beta^{(r)}=\{\beta^{(r)}_i|1\le i\le d\}\}$. The sample space is $\ml(\ma)^n$. Given a parameter $\vec\theta\in \Theta$, the probability of any linear order $R_j=[a_{i_1}\succ a_{i_2}\succ\ldots\succ a_{i_m}]$ given by agent $j$ is 
$
\Pr\nolimits_{\text{$k$-PL}(\mx)}(R_j|\vec\theta)=\sum^k_{r=1}\alpha_r\Pr\nolimits_{\plx}(R_j|\vec\beta^{(r)})$.
\end{dfn}

$\plx$ can be viewed as a special case of $k$-$\plx$ where $k=1$. We recall the definition of identifiability of statistical models as follows.

\begin{dfn}[Identifiability]\label{def:id}
Let $\mathcal{M}=\{\Pr(\cdot|\vec{\theta}): \vec{\theta}\in\Theta\}$ be a statistical model, where $\Theta$ is the parameter space and $\Pr(\cdot|\vec{\theta})$ is the distribution over the sample space associated with $\vec{\theta}\in \Theta$. We say that a parameter $\vec\theta\in\Theta$ is {\em identifiable} in $\mm$, if for any $\vec\gamma\in\Theta$ with $\vec\gamma\ne\vec\theta$, we have $\Pr(\cdot|\vec{\theta})\ne \Pr(\cdot|\vec{\gamma})$.  $\mathcal{M}$ is {\em identifiable} if all its parameters are identifiable.
\end{dfn}
The following example shows that non-identifiability of a model can lead to unavoidable untrustworthy interpretations and decisions. 

\begin{ex}
\label{ex:running}\rm
Suppose an automobile manufacturer is using $\plyz$ to learn consumers' preferences over car models. For simplicity suppose there are two agents \{1, 2\} and two alternatives (car models) $\{a_1, a_2\}$. 
$Y=[y_1, y_2] = [0.5, 1]$, where each agent is represented by her normalized income ($0.5$ for the first agent and $1$ for the second agent).  $Z = [\vec z_1, \vec z_2] = \begin{bmatrix}0.6 & 1\\0.2 & 0.5\end{bmatrix}$, where each car is represented by its normalized price ($0.6$ for the first car and $1$ for the second car) and its normalized miles per gallon ($0.2$ for the first car and $0.5$ for the second car). 

Let $B=[-1, 8/3]^\top$ and  $B'=[1, 0]^\top$. We show that $B$ and $B'$ correspond to exactly the same distribution over the two agents' preferences. In fact, it is not hard to verify that  $\vec z^\top_1By_1-\vec z^\top_2By_1=z^\top_1B'y_1-\vec z^\top_2B'y_1=-0.2$ and $\vec z^\top_1By_2-\vec z^\top_2By_2=\vec z^\top_1B'y_2-\vec z^\top_2B'y_2=-0.4$.  Therefore, $\Pr\nolimits_{\plyz}(R_1=a_1\succ a_2|B)=\frac 1 {1+\exp(\vec z^\top_2By_1-\vec z^\top_1By_1)}=\frac 1 {1+\exp(\vec z^\top_2B'y_1-\vec z^\top_1B'y_1)}=\Pr\nolimits_{\plyz}(R_1=a_1\succ a_2|B')$. Other probabilities can be calculated similarly. 

Therefore, it is impossible for any statistical method to distinguish $B$ from $B'$. This may not be a big concern if the company uses the learned model to predict the preferences of new customers, as both $B$ and $B'$ would give the same prediction. However, the first components of $B$ and $B'$ have opposite interpretations. The first component of $B$ being positive is often interpreted as the existence of a negative correlation between an agent's income and the car's price, i.e.~richer people prefer cheaper cars. The interpretation of the first component of $B'$ is on the opposite. This makes the interpretation and any decision based on it untrustworthy.  \hfill$\Box$
\end{ex}

\begin{figure}[htp]
    \centering
    \includegraphics[width = 0.45\textwidth]{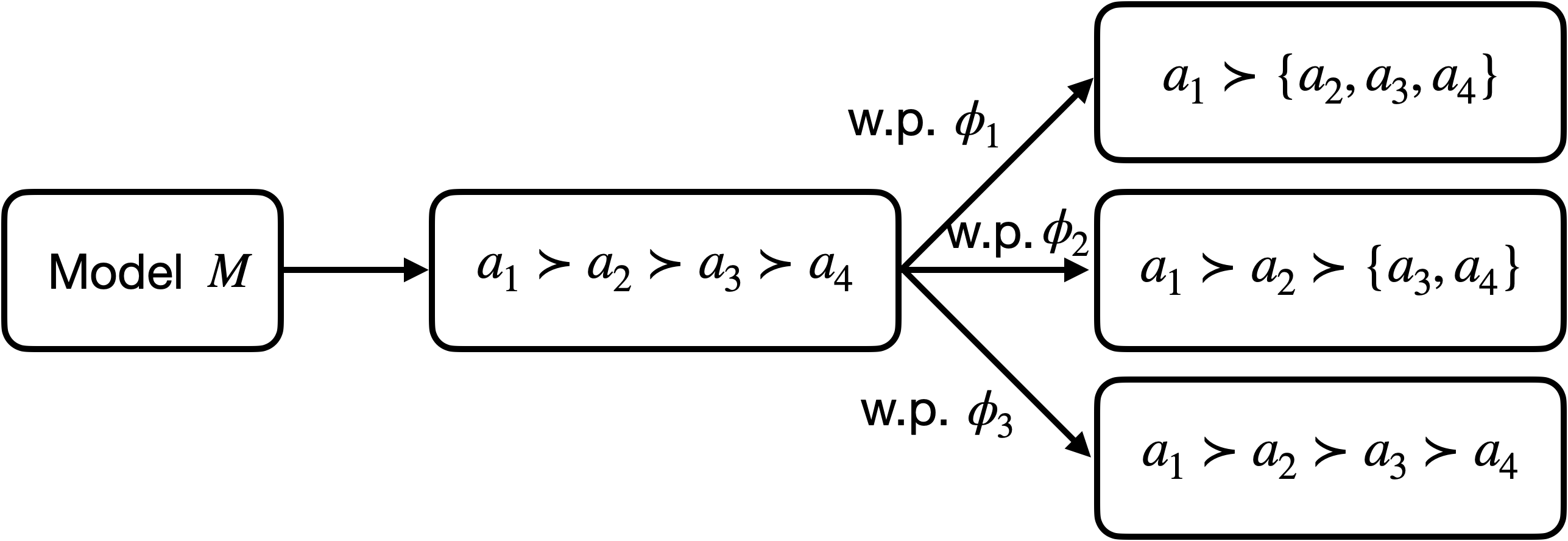}
    \caption{Illustration of model M-TO with the parameter $\vec\phi$, where M is a model that generates linear orders. M-TO allows users to report top-$l$ orders of different $l$'s.} 
    \label{fig:phi_l}
\end{figure}

\section{Models and Their Identifiability for Top-$l$ Orders}

Modeling and learning from different partial orders is desirable in the scenarios where users provide different structures of partial orders~\citep{Zhao19:Learning}. Following~\citet{Zhao19:Learning}, we extend $\plx$ to a model that generates top-$l$ partial orders by introducing a new parameter $\vec\phi=(\phi_1, \ldots, \phi_{m-1})$, where $\sum^{m-1}_{i=1}\phi_i=1$. $\vec\phi$ can be viewed as a distribution over $\{1, 2, \ldots, m-1\}$ that represents the probability for the agent to report a top-$l$ order, where $1\le l\le m-1$. 
Any model defined in the previous section can be  extended to a model that generates a top-$l$ order $R$ with probability $\Pr\nolimits_{\mm\text{-TO}}(R|\vec\theta, \vec\phi)=\phi_{l}\Pr\nolimits_{\mm}(R|\vec\theta)$, as is illustrated in Figure~\ref{fig:phi_l}. For example, $\plx$ can be extended to $\plxp$, formally defined as follows.

\begin{figure}[htp]
    \centering
    \includegraphics[width = 0.4\textwidth]{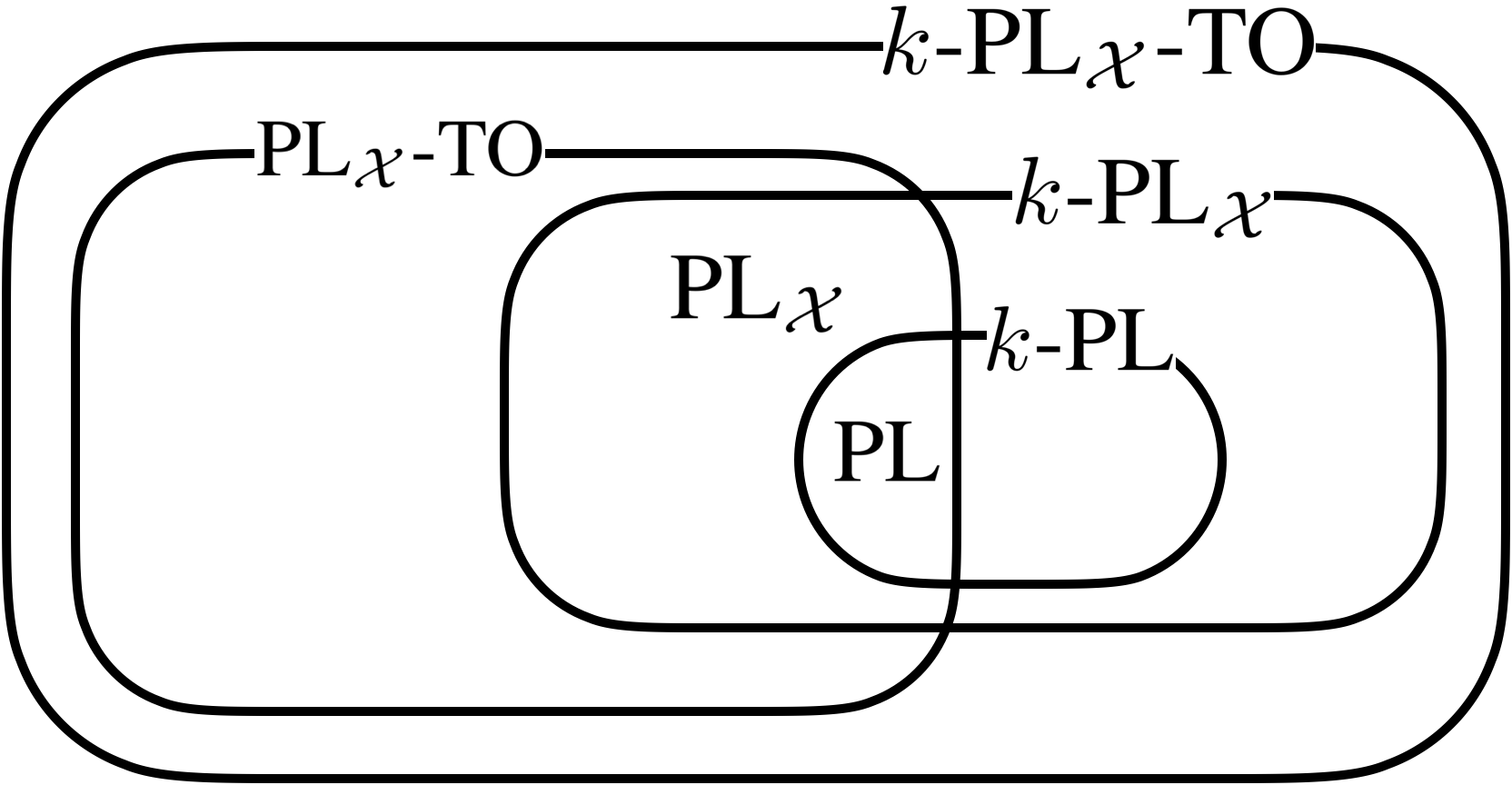}
    \caption{Relations between models in this paper in Venn diagram. $k$-$\plxp$ is the most general model, subsuming all the other models. The two largest submodels of $k$-$\plxp$ are $\plxp$ and $k$-$\plx$, whose intersection is $\plx$. PL is the smallest model in this diagram, lying at the intersection of $\plxp$ and $k$-PL.} 
    \label{fig:model_relations}
\end{figure}

\begin{dfn}[$\plxp$]\label{def:plxphi}  Let $\mx\in\mathbb R^{d\times mn}$ denote the feature matrix. The parameter space is $\Theta={\mathbb R}^d\times\{ \vec\phi\in {\mathbb R}_{\ge 0}^{m-1}: \vec\phi\cdot\vec 1 =1 \}$. The sample space is $\mt(\ma)^n$. Given a parameter $(\vec\beta, \vec\phi)$, the probability of any top-$l$ order $O_j=a_{i_1}\succ a_{i_2}\succ\ldots\succ a_{i_l}\succ\text{others}$ given by agent $j$ is 
$
\Pr\nolimits_{\plxp}(O_j|\vec\beta, \vec\phi)=\phi_{l}\Pr\nolimits_{\plx}(O_j|\vec\beta)
$, where $\Pr_{\plx}(O_j|\vec\beta)$ is the marginal probability of $O_j$ under $\plx$ given $\vec\beta$.
\end{dfn}

Again, $\mx$ is assumed given and not part of the parameter of $\plxp$. $\plx$ is a submodel of $\plxp$ (where  $\phi_{m-1}=1$). 
$\plyzp$ and $k$-$\plxp$ can be defined similarly, see Appendix A 
for their formal definitions. The relations between different models mentioned in this paper are shown in the Venn diagram in Figure~\ref{fig:model_relations} ($\plyz$ and $\plyzp$ are omitted for simplicity).

As was illustrated in Example~\ref{ex:running}, identifiability is important if one wants to interpret the learned parameter. For the rest of this section, we focus on identifiability of $\plxp$ and $k$-$\plxp$.

To characterize the identifiability of PL extensions with features, for each $j\le n$, we first define agent $j$'s {\em normalized} feature matrix, denoted by $\norm(X_j)$.
\begin{equation}\label{eq:norm}
\norm(X_j)=[\vec x_{j2}-\vec x_{j1}, 
\vec x_{j3}-\vec x_{j1}, 
\ldots, 
\vec x_{jm}-\vec x_{j1}].
\end{equation} 

We then define 
the $d$-by-$(m-1)n$ {\em normalized} feature matrix, denoted by  $\bignorm(\mx)$, as follows.
\begin{equation}\label{eq:xprime}
\bignorm(\mx)=[\norm(X_1), \norm(X_2), \ldots, \norm(X_n)]
\end{equation}
In words, $\mx$ is normalized by using the feature vector of $a_1$ as the baseline.  Our results still hold if another alternative is used as the baseline.
We now  present our first identifiability theorem.

\begin{thm}\label{thm:idplx}
For any $\mx$, 
$\plxp$ is identifiable if and only if $\bignorm(\mx)$ has full row rank. 
\end{thm}

The full proof can be found in Appendix G.2.




\begin{ex}\label{ex:plx}
Consider a $\plx$ (a special case of $\plxp$), whose feature matrix $\mx$ has three rows $\vec r_1$, $\vec r_2$, and $\vec r_3$, where $\vec r_1+\vec r_2 = \vec r_3$. Therefore, $\bignorm(X)$ does not have full row rank. Let $\vec\beta = [\beta_1, \beta_2, \beta_3]^\top$ be the ground truth parameter. We construct $\vec\beta'=[\beta_1+\beta_3, \beta_2+\beta_3, 0]^\top$. Then it is easy to see $\vec\beta^\top\cdot\mx=\vec\beta'^\top\cdot\mx$, which further means for any order $R$, we have $\Pr\nolimits_{\plx}(R|\vec\beta)=\Pr\nolimits_{\plx}(R|\vec\beta')$ by Definition~\ref{def:fpl}. This means this $\plx$ is not identifiable.
\end{ex}

Theorem~\ref{thm:idplx} can be applied to characterize the identifiability for PL with bilinear features as in the following corollary. 

\begin{coro}\label{coro:idbilinear}
For any model $\plyzp$, where $Y\in\mathbb R^{L\times n}$ and $Z\in\mathbb R^{K\times m}$, $\plyzp$ is identifiable if and only if both $Y$ and $\norm(Z)$ have full row rank.
\end{coro}

The formal proof can be found in Appendix G.3.

The full row rank condition in Theorem~\ref{thm:idplx} and Corollary~\ref{coro:idbilinear} is mild as long as $n$ is not too small, which is the case in many real-world applications. For example, the full row rank condition holds on the sushi dataset used in our real-world experiment. In Appendix F.4, 
we show that the probability of violating the full row rank condition becomes very small when sampling $n=10$ agents from the sushi dataset, and this probability decays exponentially as $n$ increases.

In the next example, we show that the sufficient conditions for $\plyzp$ to be identifiable by~\citet[Proposition 1]{Schafer18:Dyad} is unfortunately flawed.

\begin{ex}
\label{ex:bilinear}
Continuing Example~\ref{ex:running}, it is not hard to verify that no agent feature or alternative feature is a constant, which implies that $\plyz$ is identifiable according to~\citep[Proposition 1]{Schafer18:Dyad}. However, as we showed in Example~\ref{ex:running}, $\plyz$ is not identifiable. It's easy to see that $\norm(Z)$ has two rows but only one column, which does not have full row rank.
\end{ex}

Identifiability of mixtures of PLs with features is at least as challenging as the identifiability of mixtures of standard PLs, which is still an open problem for any $k\ge 3$. In the following theorem, we provide a sufficient condition for a parameter in $k$-$\plxp$ to be identifiable.



\begin{thm}\label{thm:idkplx}
If $k$-PL is identifiable, then for any $\mx$ such that $\bignorm(\mx)$ has full row rank,  any parameter $(\vec\alpha, \vec\beta, \vec\phi)$ with $\phi_{m-1}>0$ is identifiable in $k$-$\plxp$. 
\end{thm}
The proof is done by contradiction. The full proof is provided in Appendix G.4.

\begin{coro}
\label{coro:idkplx}
For any $\mx$ such that $\bignorm(\mx)$ has full row rank, $2$-$\plxp$ over four or more alternatives is identifiable modulo label switching.
\end{coro}

\begin{figure*}[htp]
    \centering
    \includegraphics[width = 0.45\textwidth]{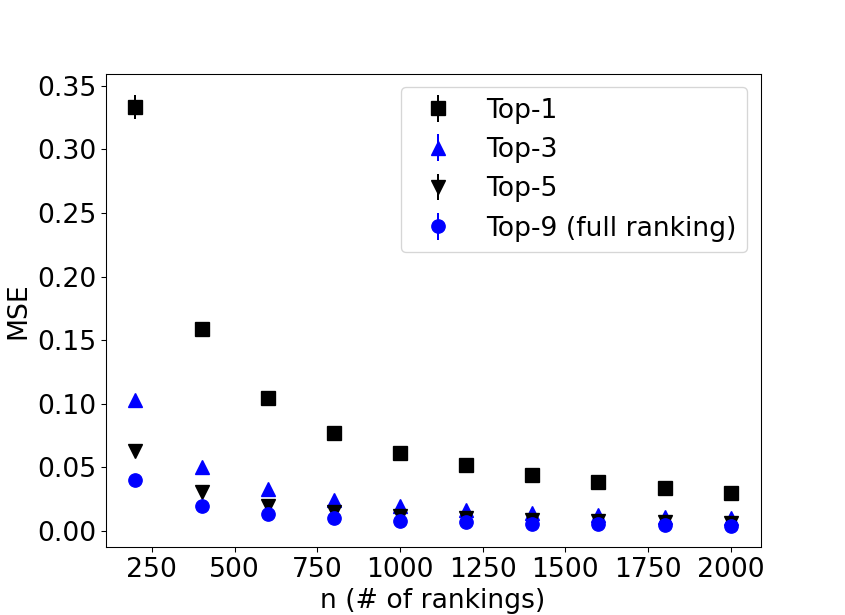}\includegraphics[width=0.45\textwidth]{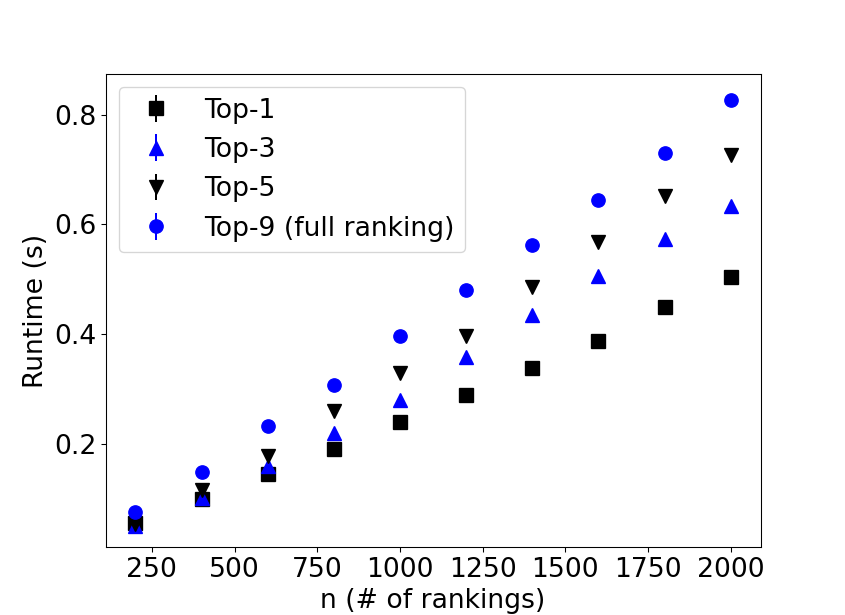}
    \vspace{-1em}
    \caption{MSE (left) and Runtime (right) with 95\% confidence intervals for MLE on $\plxp$ given top-1 only, top-3 only, top-5 only, and top-9 (full rankings) over 2000 trials. Results for Top-7 are very close to those for top-9, and therefore omitted.} 
    \label{fig:plx}
\end{figure*}

\section{MLE of $\plxp$ and Its Consistency}

Let $P=(O_1, \ldots, O_n)$ denote the input data, where for each $j\le n$, $O_j$ is a top-$l_j$ order. Let  $\mx\in\mathbb R^{d\times mn}$ denote the feature matrix. MLE of $\plxp$ computes the parameter that maximize the following log likelihood function: 
\begin{align*}
LL(P|\vec\phi, \vec\beta)&=\sum^n_{j=1}\ln\Pr\nolimits_{\plxp}(O_j|\vec\beta)\\
&=\sum^n_{j=1}(\ln\phi_{l_j}+\ln\Pr\nolimits_{\plx}(O_j|\vec\beta))
\end{align*}
Note that  $\vec\phi$ and $\vec\beta$ parameter are separated in the log likelihood function, we can compute them separately as follows.
\begin{align}
&\vec\phi^*=\arg\max_{\vec\phi}\sum^n_{j=1}\ln\phi,\text{s.t.} \sum^{m-1}_{l=1}\phi_l=1\label{eq:mlephi}\\
&
\vec\beta^*=\arg\max_{\vec\beta}\sum^n_{j=1}\ln\Pr\nolimits_{\plx}(O_j|\vec\beta)\label{eq:mle}
\end{align}

$\vec\phi$  can be easily computed by counting the frequencies of each top-$l$ order. The main challenge is to accurately estimate the $\vec\beta$ part, which is the main focus of the rest of this section.

The following theorem provides a necessary and sufficient condition for the objective function in  \eqref{eq:mle} to be strictly concave, which turns out to be the same condition for the identifiability of $\plxp$. Strict concavity is desirable because, combined with boundedness, it guarantees the convergence of MLE.

\begin{thm}\label{thm:log-concavity}
For any $\plxp$ and any data $P=(O_1, \ldots, O_n)$, the log likelihood function in \eqref{eq:mle} is strictly concave if and only if $\bignorm(\mx)$ has full row rank.
\end{thm}
The full proof can be found in Appendix G.5.

We now introduce an assumption to guarantee the boundedness of MLE for  given data $P$. Boundedness is important for consistency because a strictly concave function may not converge if it is unbounded. For any $j\le n$,  $i_1\le m$, and $i_2\le m$, we define $\xi_{j, i_1i_2}$ as follows. 
$$\xi_{j, i_1i_2}=\left\{\begin{array}{rl} 1& \text{if agent $j$ prefers $a_{i_1}$ over $a_{i_2}$}\\
-1 & \text{if agent $j$ prefers $a_{i_2}$ over $a_{i_1}$}\\
0& \text{if agent $j$'s preference between $a_{i_1}$}\\
&\text{and $a_{i_2}$ is not available}
\end{array}\right.$$ 

\begin{ass}\label{ass:bound}
Let $\mx\in\mathbb R^{d\times mn}$ denote a feature matrix and let $P$ denote the data. For any $r\le d$, there exist $j_1, j_2\in\{1, \ldots, n\}$ with $j_1\ne j_2$ and $i_1, i_2\in\{1, \ldots, m\}$ with $i_1\ne i_2$ such that $\xi_{j_1, i_1i_2}\xi_{j_2, i_1i_2}(x_{j_1i_1, r}-x_{j_1i_2, r})(x_{j_2i_1, r}-x_{j_2i_2, r})<0$.
\end{ass}

At a high level, Assumption~\ref{ass:bound} is a mild condition that requires sufficient diversity in agents' preferences, which mirrors \citeauthor{Hunter04:MM}'s assumption  for PL~\citep[Assumption 1]{Hunter04:MM}.

The following lemma shows that
Assumption~\ref{ass:bound} is sufficient for MLE to be bounded.

\begin{lem}\label{lem:bound}
For any $\plxp$ and data $P$, if Assumption~\ref{ass:bound} holds then the MLE in~\eqref{eq:mle} is bounded.
\end{lem}
The proof is provided in Appendix G.6. 
Finally, the following theorem provides a bound on the RMSE of MLE for $\plxp$ given that $\bignorm(\mx)$ has full row rank.






\begin{thm}\label{thm:msebound}
Given any $\plxp$ over $m$ alternatives and $n$ agents with the feature matrix $\mx\in\mathbb R^{d\times mn}$. Define $L(\vec\beta)=\frac 1 n \sum^n_{j=1}\ln\Pr\nolimits_{\plx}(O_j|\vec\beta)$, which is $\frac 1 n$ of the objective function in \eqref{eq:mle}. Let $H(\vec\beta)$ denote the Hessian matrix of $L(\vec\beta)$ and $\lambda_1(\vec\beta)$ be the smallest eigenvalue of $-H(\vec\beta)$. Let $\vec\beta_0$ denote the ground truth parameter and $\vec\beta^*$ denote the estimated parameter that is computed using \eqref{eq:mle}. Define $\lambda_{\min}=\min_{0\le\sigma\le 1}\lambda_1(\sigma\vec\beta^*+(1-\sigma)\vec\beta^0)$.

If \bignorm($\mx$) has full row rank and Assumption~\ref{ass:bound} holds, then for any $0<\delta<1$, with probability $1-\delta$, 
\begin{equation}\label{eq:samplecomplexity}
||\vec\beta^*-\vec\beta_0||_2\le\frac {\sqrt{8(m-1)^2c^2d\ln(\frac {2d} {\delta})}} {\lambda_{\min}\sqrt{n}},
\end{equation}
where $c$ is the difference between the largest and the smallest entries of $\mx$.
\end{thm}

The full proof is in Appendix G.7.

The RMSE (root-mean-square-error) bound given by Theorem~\ref{thm:msebound} is not tight since we make no assumptions on the distribution of features. While Theorem~\ref{thm:msebound} does provide insights on convergence of MLE:\\
\noindent 1. {\em Consistency}: as $n$ increases, RMSE $||\vec\beta_0-\vec\beta^*||_2$ decreases at the rate of $\frac 1 {\sqrt{n}}$. When $n$ approaches infinity, RMSE approaches $0$.\\
\noindent 2. {\em Sample complexity}: for any $\epsilon>0$, $0<\delta<1$, $\Pr(||\vec\beta^*-\vec\beta_0||\le\epsilon)\ge 1-\delta$ when $n\ge \frac {8(m-1)^2c^2d\ln\frac {2d} {\delta}} {\lambda_{\min}^2\epsilon^2}$. This is obtained by letting $\epsilon\ge\frac {\sqrt{8(m-1)^2c^2d\ln(\frac {2d} {\delta})}} {\lambda_{\min}\sqrt{n}}$.\\
\noindent 3. {\em Approximation of $\lambda_{\min}$}: in practice, when the size of data is not too small, $\lambda_{\min}$ can be approximated by $\lambda_1(\vec\beta^*)$ because $\vec\beta^*$ approaches $\vec\beta^0$ as $n$ increases. This gives a practical way of computing the RMSE bound when the ground truth is unknown.




\section{Experiments}

We show experiments on synthetic data in this section and provide additional experiments on mixture models and on real-world data in the appendix.

{\noindent\bf Setup.} Fix $m=10$ and $d=10$. For each agent and each alternative, the feature vector is  generated in $[-1,1]$ uniformly at random. Each component in $\vec\beta$ is generated uniformly at random in $[-2, 2]$. MLE for $\plxp$ was implemented in MATLAB with the built-in \texttt{fminunc} function and tested on a Ubuntu Linux server with Intel Xeon
E5 v3 CPUs each clocked at 3.50 GHz. We use mean squared error (MSE) and runtime to measure the statistical efficiency and computational efficiency of algorithms, respectively. Results are shown in Figure~\ref{fig:plx}. All values are computed by averaging over $2000$ trials.

\noindent{\bf Observations.} 
Figure~\ref{fig:plx} shows the performance of MLE for $\plxp$. We observe that MSE decreases as the number of agents increases, which demonstrates consistency of MLE for $\plxp$. Moreover, learning from top-$l$ preferences with different $l$ values provides tradeoffs between statistical efficiency and computational efficiency. 

\section{Summary and Future Work}
We provide the first set of theoretical results on the identifiability of mixtures of PL with features for top-$l$ preferences. We also identify conditions for the MLE of $\plxp$ to be consistent, and propose an EM algorithm to handle general $k$-$\plxp$. In the full version of this paper~\citep{zhao20:learning}, we provide a generalized Rank-Breaking then Composite Marginal Likelihood algorithm for learning RUMs beyond PL from incomplete preferences and show its performance on synthetic data. \citep{zhao20:learning} also includes additional experiments on mixture models as well as missing proofs. 
Generic identifiability and efficient algorithms for $k$-$\plxp$ are natural questions for future work.

%

\section*{Acknowledgements}
We thank anonymous reviewers for helpful comments. Lirong Xia is supported by NSF \#1453542 and a gift fund from Google. Ao Liu is supported by an RPI-IBM AI Horizons scholarship.

{\small
\bibliography{references}
\bibliographystyle{named}
}

\appendix
\clearpage

\begin{figure*}[htp]
    \centering
    \includegraphics[width = 0.5\textwidth]{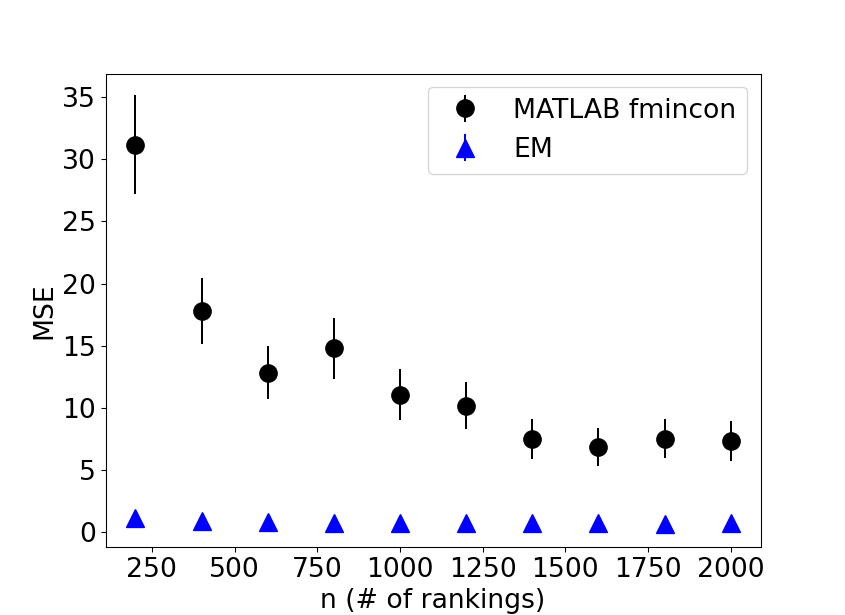}\includegraphics[width=0.5\textwidth]{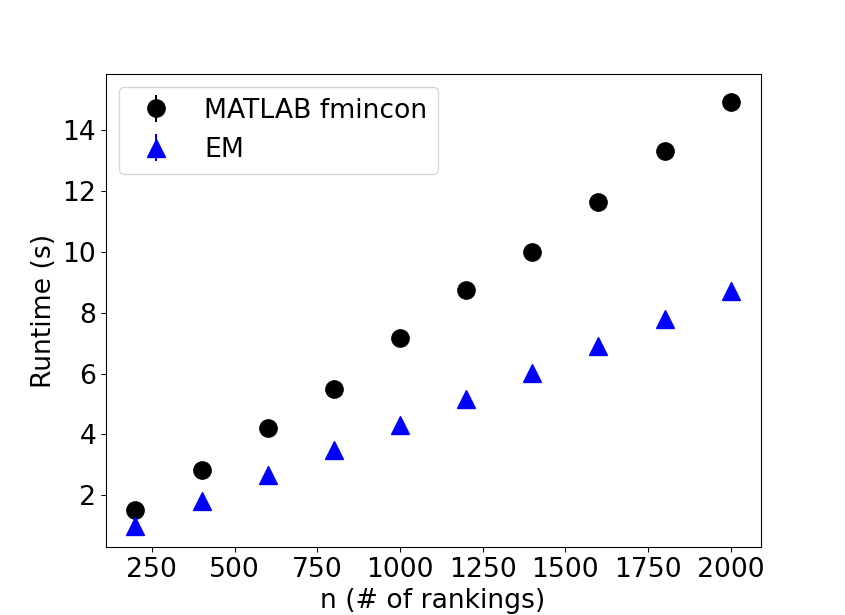}
    \caption{Comparisons of MSE and running time with 95\% confidence intervals between maximizing the likelihood function using a generic Matlab function \texttt{fmincon} and 10-iteration EM for k-$\plx$ over $2000$ trials.}
    \label{fig:2plx}
\end{figure*}

\section{Additional Definitions of Ranking Models}
\label{sec:dfn}

\begin{dfn}[Random utility models (RUMs)] The random utility model over $\ma$ associates each alternative $a_i$ with a utility distribution $\mu_i$. The parameter space is $\Theta = \{\vec\theta=\{\vec\theta_i|i=1, 2, \ldots, m\}\}$, where $\vec\theta_i$ is the parameter for the utility distribution $\mu_i$ corresponding to alternative $a_i$. The sample space is $\ml(\ma)^n$. A linear order is generated in two steps. First, for each $i\le m$, a latent utility $u_i$ is generated from $\mu_i(\cdot|\vec\theta_i)$ independently; second, the alternatives are ranked according to their utilities in the descending order. Given a parameter $\vec\theta$, the probability of generating a linear order $R=a_{i_1}\succ a_{i_2}\succ \ldots\succ a_{i_m}$ is
\begin{align*}
\Pr\nolimits_{\text{RUM}}(R|\vec{\theta})&=\int^\infty_{-\infty}\int^\infty_{u_{i_m}}\cdots\int^\infty_{u_{i_2}}\mu_{i_m}(u_{i_m}|\vec\theta_{i_m})\\
&\cdots\mu_{i_1}(u_{i_1}|\vec\theta_{i_1})du_{i_1}du_{i_2}\cdots du_{i_m}
\end{align*}
\end{dfn}

PL is a special case of RUM where $\mu_i(\cdot|\theta_i)$ is the Gumbel distribution $\mu_i(x_i|\theta_i)=e^{-(x_i-\theta_i)-e^{-(x_i-\theta_i)}}$.

\begin{dfn} (RUM with features ($\rumx$)). The parameter space is $\Theta=\{\vec\beta=\{\beta_i|1\le i\le d\}\}$. For any $1\le j\le n$ and $1\le i\le m$, the utility distribution for agent $j$, alternative $a_i$ is parameterized by $\vec\beta\cdot\vec x_{ji}$ as its mean. The sample space is $\ml(\ma)^n$. Given a parameter $\vec\beta\in B$, the probability of any linear order $R_j=[a_{i_1}\succ a_{i_2}\succ\ldots\succ a_{i_m}]$ given by agent $j$ is 
\begin{align*}
\Pr\nolimits_{\rumx}(R_j|\vec{\beta})&=\int^\infty_{-\infty}\int^\infty_{u_{i_m}}\cdots\int^\infty_{u_{i_2}}\mu_{i_m}(u_{i_m}|\vec\beta_{i_m}\cdot\vec x_{jm})\\
&\cdots
\mu_{i_1}(u_{i_1}|\vec\beta_{i_1}\cdot\vec x_{ji_1})du_{i_1}du_{i_2}\cdots du_{i_m}
\end{align*}
\end{dfn}

\begin{dfn} (RUM with features for $l$-way orders ($\rumxp$)). The parameter space is $\Theta=\{\vec\beta=\{\beta_i|1\le i\le d\}\}$. For any $1\le j\le n$ and $1\le i\le m$, the utility distribution for agent $j$, alternative $a_i$ is parameterized by $\vec\beta\cdot\vec x_{ji}$ as its mean. The sample space is $\mi(\ma)$. Given a parameter $\vec\beta\in B$, the probability of any $l$-way order $R_j=[a_{i_1}\succ a_{i_2}\succ\ldots\succ a_{i_l}]$ given by agent $j$ is 
\begin{align*}
&\Pr\nolimits_{\rumx}(R_j|\vec{\beta})=p^l(1-p)^{(m-l)}\int^\infty_{-\infty}\int^\infty_{u_{i_l}}\cdots\int^\infty_{u_{i_2}}\\
&\mu_{i_l}(u_{i_l}|\vec\beta_{i_l}\cdot\vec x_{jl})\cdots
\mu_{i_1}(u_{i_1}|\vec\beta_{i_1}\cdot\vec x_{ji_1})du_{i_1}du_{i_2}\cdots du_{i_l},
\end{align*}
where $0<p\le 1$.
\end{dfn}

The above definition implies a two-step partial order generation procedure: (1) sample the subset of alternatives where each alternative is selected with probability $p$; (2) generate a linear order over the sampled subset of alternatives. 

\begin{dfn} (Mixtures of $k$ RUMs with features ($k$-$\rumx$)). The parameter space has two parts. The first part is the vector of mixing coefficients $\vec\alpha=(\alpha_1, \alpha_2, \ldots, \alpha_k)$ and the second part is $(\vec\beta^{(1)}, \vec\beta^{(2)}, \ldots, \vec\beta^{(k)})$, where for $r=1, \ldots, k$ $\vec\beta^{(r)}=\{\beta^{(r)}_i|1\le i\le d\}\}$. The sample space is $\ml(\ma)^n$. Given a parameter $\vec\beta\in B$, the probability of any linear order $R_j=[a_{i_1}\succ a_{i_2}\succ\ldots\succ a_{i_m}]$ given by agent $j$ is 
$$
\Pr\nolimits_{k-\rumx}(R_j|\vec\beta)=\sum^k_{r=1}\alpha_r\Pr\nolimits_{\rumx}(R_j|\vec\beta^{(r)}).$$
\end{dfn}

\begin{dfn}\label{def:bfplphi} ($\plyzp$). Let $Y\in\mathbb R^{L\times n}$ denote an agent feature matrix and let $Z\in\mathbb R^{K\times m}$ denote an alternative feature matrix. The parameter space is $\Theta={\mathbb R}^{K\times L}\times\{ \vec\phi\in {\mathbb R}_{\ge 0}^{m-1}: \vec\phi\cdot\vec 1 =1 \}$. The sample space is $\mt(\ma)^n$. Given a parameter $(B, \vec\phi)$, the probability of any top-$l_j$ order $O_j=a_{i_1}\succ a_{i_2}\succ\ldots\succ a_{i_{l_j}}\succ\text{others}$ given by agent $j$ is 
$
\Pr\nolimits_{\plyzp}(O_j|B, \vec\phi)=\phi_{l_j}\prod^{l_j}_{p=1}\frac {\exp(\vec z^\top_p B\vec y_j)} {\sum^m_{q=p}\exp(\vec z^\top_q B\vec y_j)}.$
\end{dfn}

\begin{dfn} ($k$-$\plxp$). Let $\mx\in\mathbb R^{d\times mn}$ denote a feature matrix. The parameter space $\Theta$ has three parts. The first part is the vector of mixing coefficients $\vec\alpha=(\alpha_1, \alpha_2, \ldots, \alpha_k)$; the second part is $(\vec\beta^{(1)}, \vec\beta^{(2)}, \ldots, \vec\beta^{(k)})$, where for $r=1, \ldots, k$, $\vec\beta^{(r)}=\{\beta^{(r)}_i|1\le i\le d\}\}$; and the third part is $\vec\phi=(\phi_{1}, \ldots, \phi_{m-1})$, where $\phi_{1}, \ldots, \phi_{m-1}\ge 0, \sum^{m-1}_{l=1}\phi_l=1$. The sample space is $\mt(\ma)^n$. Given a parameter $\vec\theta\in \Theta$, the probability of any top $l_j$ order $O_j=[a_{i_1}\succ a_{i_2}\succ\ldots\succ a_{i_{l_j}}]$ given by agent $j$ is 
$$
\Pr\nolimits_{k-\plxp}(O_j|\vec\theta)=\phi_{l_j}\sum^k_{r=1}\alpha_r\Pr\nolimits_{\plx}(O_j|\vec\beta^{(r)}).$$
\end{dfn}

\section{Identifiability of $\rumx$}

\begin{thm}\label{thm:idrumx}
Given $0<p\le 1$, for any $\mx$, $\rumxp$ is identifiable if and only if $p>0$ and $\bignorm(\mx)$ has full row rank. 
\end{thm}
\begin{proof} Let $f(\vec\beta\cdot\vec x_{ji_1}-\vec\beta\cdot\vec x_{ji_2})$ denote the probability of $a_{i_1}\succ a_{i_2}$ by agent $j$.

``if" direction. We prove that if $\rumxp$ is not identifiable, then $\bignorm(\mx)$ does not have full row rank. Since $\rumxp$ is not identifiable, there exist $\vec\beta^{(1)}\ne\vec\beta^{(2)}$ leading to the same distribution of rankings. Then for any $j$ and any $i_1\ne i_2$, we have 

\begin{equation}\label{eq:diff}
f({\vec\beta^{(1)}\cdot\vec x_{ji_1}}-\vec\beta^{(1)}\cdot\vec x_{ji_2})=f({\vec\beta^{(2)}\cdot\vec x_{ji_1}}-\vec\beta^{(2)}\cdot\vec x_{ji_2}),
\end{equation}
Due to monotonicity of $f$, we have $$(\vec\beta^{(1)}-\vec\beta^{(2)})(\vec x_{ji_1}-\vec x_{ji_2})=0.$$

There are $n(m-1)$ independent such equations, corresponding the rows in $\bignorm(\mx)$. Since there exists nonzero $\vec\beta^{(1)}-\vec\beta^{(2)}$ s.t. $\bignorm(\mx)\cdot(\vec\beta^{(1)}-\vec\beta^{(2)})=0$, $\bignorm(\mx)$ does not have full row rank.

``only if" direction. For the purpose of contradiction suppose $\bignorm(\mx)$ does not have full row rank, then there exists $\beta^{(1)}\ne\beta^{(2)}$ s.t. $\bignorm(\mx)\cdot(\vec\beta^{(1)}-\vec\beta^{(2)})=0$, which means for any agent $j$ and any two alternatives $a_1$ and $a_2$, $\vec\beta^{(1)}(\vec x_{ji_1}-\vec x_{ji_2})=\vec\beta^{(2)}(\vec x_{ji_1}-\vec x_{ji_2})$ holds. Let $\sigma=\vec\beta^{(1)}\cdot\vec x_{j1}-\vec\beta^{(2)}\cdot\vec x_{j1}$. Then for any $i=1, \ldots, m$, $\vec\beta^{(1)}\cdot\vec x_{j1}-\vec\beta^{(2)}\cdot\vec x_{j1}=\sigma$. This means $\vec\beta^{(1)}$ and $\vec\beta^{(2)}$ lead to exactly the same distribution of rankings. The model is not identifiable, which is a contradiction.
\end{proof}

This theorem applies to PL with features as well since PL is a special case of RUM. 

\section{MLE for Learning $k$-$\plx$}

MLE algorithm for $k$-$\plx$ is straightforward. We compute $
\vec\alpha$ and $\vec\beta^{(1)}, \ldots, \vec\beta^{(k)}$ by maximizing the log-likelihood function
\begin{align*}
&(\vec\alpha', \vec\beta'^{(1)}, \ldots, \vec\beta'^{(k)})\\
=&\arg\max_{\vec\alpha, \vec\beta^{(1)}, \ldots, \vec\beta^{(k)}}\sum^n_{j=1}\Pr\nolimits_{k\text{-}\plx}(R_j|\vec\alpha, \vec\beta^{(1)}, \ldots, \vec\beta^{(k)})
\end{align*}

\begin{algorithm}
{\bf Input}: Preference profile $P$ with $n$ orders; feature matrix $\mx$; number of iterations $T$\\
{\bf Output}: Mixing coefficients $\vec\alpha$ and $k$ components $\vec\beta^{(1)}, \ldots, \vec\beta^{(k)}$\\
{\bf Initialization}: Randomly generate $\vec\alpha^{(0)}$ and $\vec\beta^{(1, 0)},\ldots,\vec\beta^{(k, 0)}$
\begin{algorithmic}
\FOR{$t = 1$ to $T$}
\STATE E-step:   Compute $w^{(t)}_{jr}$ using \eqref{elsr:estep} for all $j=1,2,\ldots,n$ and $r = 1,2,\ldots,k$.
\STATE M-step:
\STATE Compute $\vec\alpha^{(t)}$ using \eqref{elsr:mstep}\\
\FOR{$r = 1$ to $k$}
\STATE Compute $\vec\beta^{(r, t)}$ using \eqref{eq:mle}.
\ENDFOR
\ENDFOR
\end{algorithmic}
\caption{EM algorithm for $k$-$\plxp$.}
\label{alg:em}
\end{algorithm}

\section{An EM Algorithm for $k$-$\plxp$}
In this section we propose a natural EM algorithm for $k$-$\plxp$. Let $z_{jr}$ denote the membership indicator where $z_{jr}=1$ means the order $O_j$ belongs to the $r$th component. Let $w_{jr}$ denote the weight of order $O_j$ in $r$th component. For all $j$, we have $\sum_r w_{jr}=1$. Given previous iteration estimate $(\vec\alpha^{(t-1)}, \vec\beta^{(1, t-1)}, \ldots, \vec\beta^{(k, t-1)})$, each order $O_j$ is clustered to each component as follows.
\begin{align}\label{elsr:estep}
w^{(t)}_{jr}&=\Pr(z_{jr}=1|O_j, \vec\beta^{(r, t-1)})\notag\\
&=\frac {\Pr\nolimits_{\plx}(O_j|\vec\beta^{(r, t-1)})\cdot\alpha^{(t-1)}_r} {\sum^k_{s=1}\Pr\nolimits_{\plx}(O_j|\vec\beta^{(s, t-1)})\cdot\alpha^{(t-1)}_s}
\end{align}
Then in the M step, we have
\begin{equation}\label{elsr:mstep}
\alpha^{(t)}_r = \frac {\sum_j w^{(t)}_{jr}} {n}
\end{equation}
$\vec\beta^{(r, t)}$'s are computed using MLE for $\plxp$. The algorithm is formally shown in Algorithm~\ref{alg:em}.

\section{Generalized RBCML for $\rumxp$}

\subsection{Rank Breaking}

Rank breaking is to obtain a set of (weighted) pairwise comparisons from full rankings. For example, from $\{a_2\succ a_1\succ a_3, a_1\succ a_2\succ a_3\}$, we can obtain $\{a_2\succ a_1, 2\times (a_2\succ a_3), 2\times (a_1\succ a_3), a_1\succ a_2\}$ using uniform breaking. Given a rank-breaking and data, we construct a weighted directed graph $G$, whose vertices are the alternatives and the weight of the edge from $a_{i_1}$ to $a_{i_2}$ is the frequency of $a_{i_1}\succ a_{i_2}$ in the data. We denote this frequency by $\kappa_{i_1i_2}$. $G$ can also be represented by its adjacency matrix $K$, whose diagonal entries are zeros and $(i_1, i_2)$ entry is $\kappa_{i_1i_2}$ for all $i_1\ne i_2$. In the example of $\{a_2\succ a_1\succ a_3, a_1\succ a_2\succ a_3\}$, we have 
$$
K=\begin{bmatrix}
0 & 1 & 2\\
1 & 0 & 2\\
0 & 0 & 0
\end{bmatrix}
$$

For $l$-way orders, we let $w(l)$ denote a weighting function. Define
$$
X_{a_{i_1}\succ a_{i_2}}(R)=
\begin{cases}
w(l) & \text{if } a_{i_1}\succ a_{i_2} \text{ in } R\\
0 & \text{otherwise}
\end{cases}
$$
where $l$ is the length of $R$. Given the data $P$, we let
\begin{equation}\label{eq:kappa}
\kappa_{i_1i_2}=\sum_{R\in P}X_{a_{i_1}\succ a_{i_2}}(R).
\end{equation}
Further,  we define $\bar\kappa_{i_1i_2}$ to be the expectation of $\kappa_{i_1i_2}$ given one ranking. Formally, 
\begin{equation}\label{eq:barkappadef}
\bar\kappa_{i_1i_2}=E[X_{a_{i_1}\succ a_{i_2}}(R)]=\lim_{n\rightarrow\infty}\frac {\sum_{R\in P}X_{a_{i_1}\succ a_{i_2}}(R)} n
\end{equation}

\subsection{Generalized RBCML}

Given $K$, which is a function of the data $P$, the RBCML framework for RUMs is the maximizer of composite log-marginal likelihood, which is defined below.
\begin{dfn}[Composite marginal likelihood for RUMs] Given an RUM $\mm$, for any preference profile $P$ and any $\theta$, let $p_{i_1i_2}(\vec\theta)=\Pr_\mm(a_{i_1}\succ a_{i_2}|\vec\theta)$. The composite marginal likelihood is
\begin{equation}\label{eq:clrum}
\cl(\vec\theta,P)=\prod_{i_1\neq i_2}(p_{i_1i_2}(\vec\theta))^{\kappa_{i_1i_2}},
\end{equation}
where $\kappa_{i_1i_2}$ is the $(i_1, i_2)$ entry of matrix $K$. The composite log-marginal likelihood becomes: 
\begin{equation}\label{eq:cllrum}
\cll(\vec\theta,P) =\sum_{i_1\neq i_2}\kappa_{i_1i_2}\ln p_{i_1i_2}(\vec\theta)
\end{equation}
\end{dfn}

Then estimate of the parameter is

\begin{equation}\label{eq:cml}
\vec\theta^*=\arg\max_{\vec\theta}\cll(\vec\theta,P)
\end{equation}

The proposed generalized RBCML is formally shown as
\begin{algorithm}[H]
\caption{Generalized RBCML}
\label{alg:rbcml}
{\bf Input}: Profile $P$ of $n$ rankings. Function $w(l)$.\\
{\bf Output}: Estimated parameter $\vec\theta^*$.\\
{\bf Initialize} Randomly initialize $\vec\theta^{(0)}$
\begin{algorithmic}[1]
\STATE For all $i_1\ne i_2$, compute $\kappa_{i_1i_2}$ from $P$ using \eqref{eq:kappa}.
\STATE Compute $\vec\theta^*$ using \eqref{eq:cml}.
\end{algorithmic}
\end{algorithm}

\section{Additional Experiments}

This section provides additional experiment results on synthetic data and real-world data. 

\begin{figure*}[htp]
    \centering
    \includegraphics[width = 0.985\textwidth]{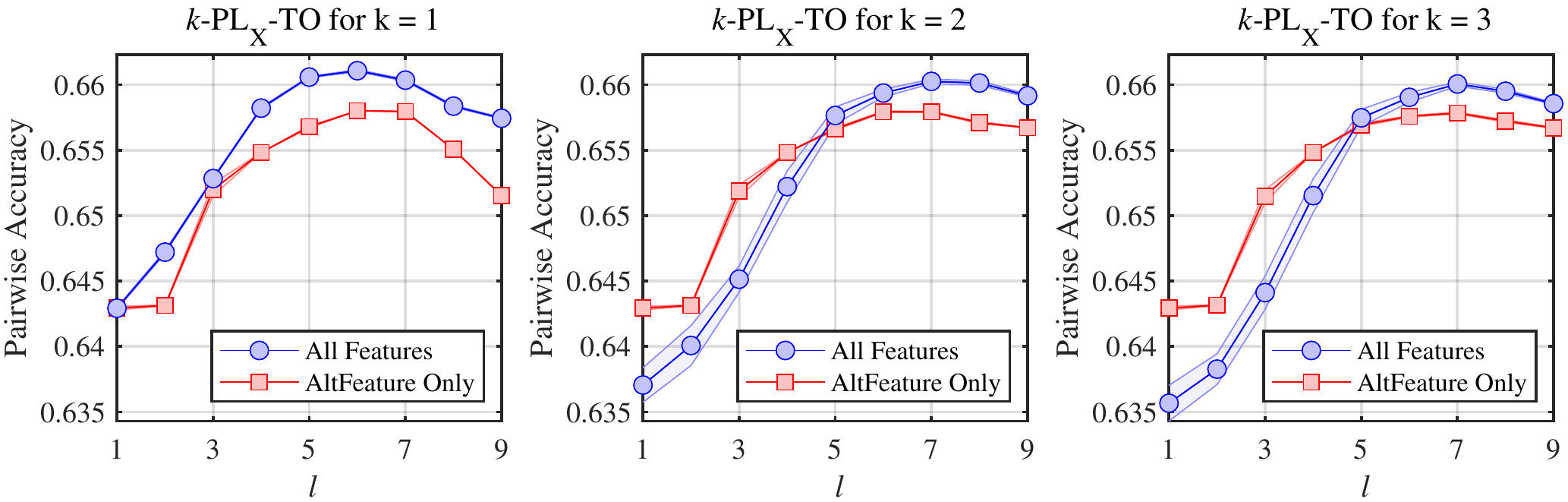}
    \vspace{-1em}
    \caption{The pairwise accuracy with 95\% confidence intervals for $k$-$\plxp$ given top-$l$ rankings on sushi dataset (top-9 means full rankings). The confidence intervals sometimes are too narrow to see (\emph{e.g.}, most points in the plot of $k = 1$.). Values are averaged over 26 5-fold cross validations.} 
    \label{fig:sushi_acc}
\end{figure*}

\subsection{Learning $k$-$\plx$ from Real-World Data}\label{sec:real_world}
\paragraph{Setup.} We learned $k$-$\plxp$ from the real-world sushi dataset~\cite{kamishima2003nantonac}, which consists of 5000 rankings over $m = 10$ alternatives. Each agent has $L = 4$ features and each alternative has $K = 4$ features. We normalized all features and learned a $k$-$\plxp$ with alternative features only (``AltFeatures Only") and $k$-$\plyzp$ (``All Features") from the data. We use Algorithm~\ref{alg:em} for $k=2, 3$, with $T=50$ fixed. For both settings, we run 5-fold cross-validations, where the training set has 4000 top-$l$ rankings while the test set has the remaining 1000 rankings. We measure the prediction accuracy using {\em pairwise accuracy}, which is the rate of correctly predicted pairwise comparisons in the test set, formally defined in Appendix~\ref{sec:accuracy}. 

\noindent{\bf Observations.} Figure~\ref{fig:sushi_acc} shows the pairwise accuracy of $k$-$\plxp$ for $k=1,2,3$ when learned from top-$l$ orders. We observe an improvement in prediction accuracy when agent features are considered in most cases (all $l$'s for $k=1$, $l\ge 5$ for $k=2, 3$). When $l$ is small and $k$ is large, agent features may harm the accuracy due to over-fitting. 
For both models, the pairwise accuracy peaks at $l=6$ or $7$. 
For larger $l$, the bottom-ranked alternatives are usually noisy because agents do not have strong preferences over alternatives they dislike~\cite{rigby2006modeling}. 

\subsection{Learning $k$-$\plx$ from Synthetic Data}

\begin{figure*}[htp]
    \centering
    \includegraphics[width = 0.7\textwidth]{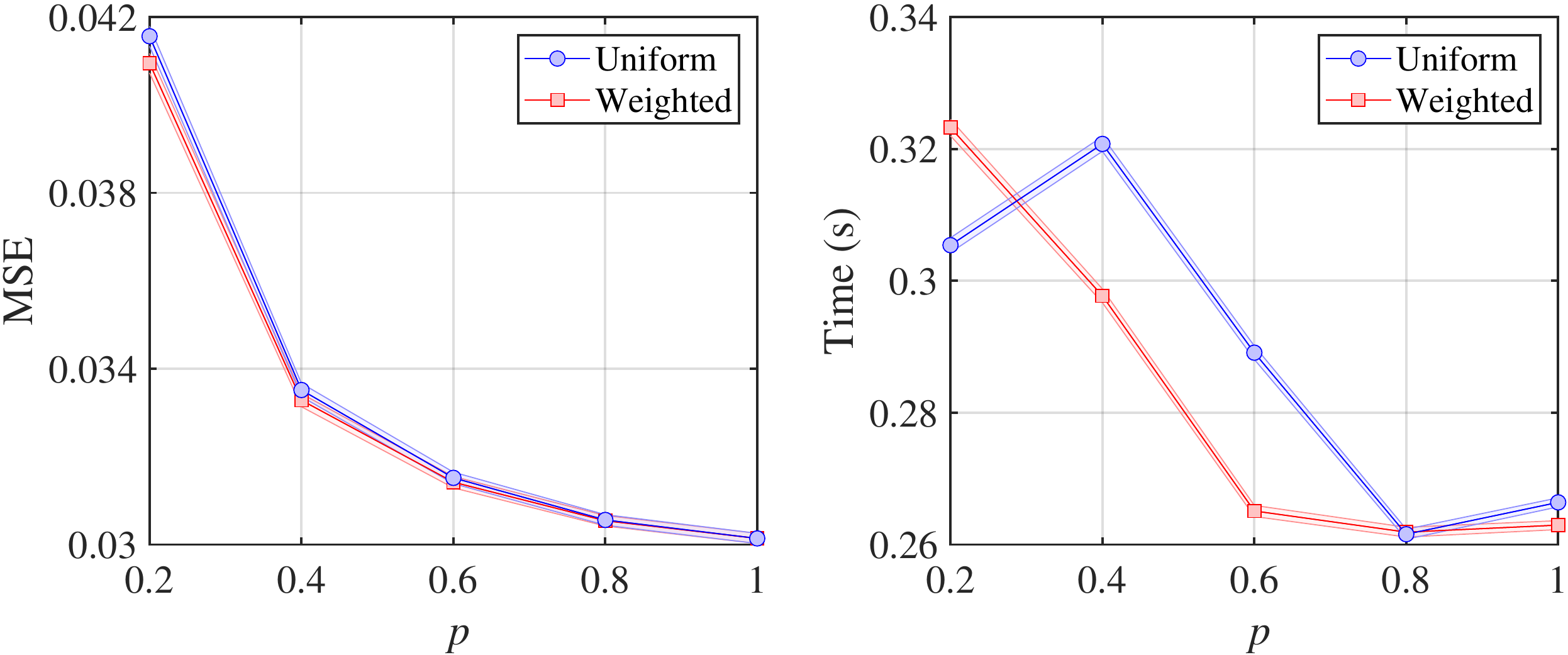}
    \caption{The MSE and running time with 95\% confidence intervals for $\text{RUM}_{\mathcal{X}}$ given $l$-way rankings on synthetic dataset of $d=6$ alternative features.} 
    \label{fig:RUM_feature}
\end{figure*}

{\noindent\bf Setup.} We still fix $m=10$ and $d=10$. For each agent and each alternative, the feature vector is generated in $[-1,1]$ uniformly at random. Each component in $\vec\beta$ is generated uniformly at random in $[-2, 2]$. For $2$-$\plx$, each component of the mixing coefficients $\vec\alpha$ is generated uniformly at random and then normalized. 

All algorithms were implemented in MATLAB (with the built-in \texttt{fmincon} for MLE for $k$-$\plx$) and tested on a Ubuntu Linux server with Intel Xeon
E5 v3 CPUs each clocked at 3.50 GHz. Results are shown in Figure~\ref{fig:2plx}. Values are computed by averaging over $2000$ trials.

\paragraph{Observations.} 
Figure~\ref{fig:2plx} illustrates the comparison between MLE computed using MATLAB \texttt{fmincon} function and the EM algorithm (Algorithm~\ref{alg:em}) for $2$-$\plx$. We observe that EM outperforms MLE w.r.t.~both statistical efficiency and computational efficiency. Therefore, EM might be a more favorable choice in practice, despite that no theoretical guarantee about its convergence is known.

\subsection{Additional Settings and Results from Real-World Experiments}

\begin{figure*}[htp]
    \centering
    \includegraphics[width = 0.985\textwidth]{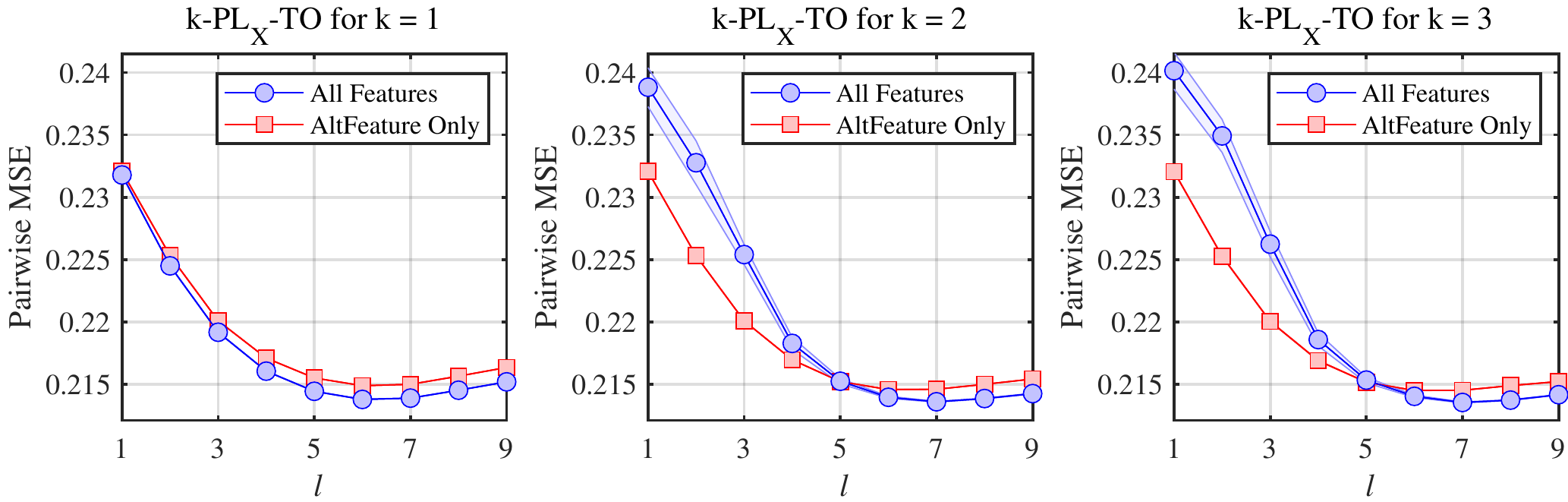}
    \vspace{-1em}
    \caption{The pairwise MSE with 95\% confidence intervals for $k$-$\plxp$ given top-$l$ rankings on sushi dataset (top-9 means full rankings).} 
    \label{fig:sushi_mse}
\end{figure*}

\subsubsection{Detailed settings}
As mentioned in Section~\ref{sec:real_world}, we used four kinds of agent features and four kinds of alternative features in sushi dataset. The alternative features are
\begin{enumerate}
\item the heaviness/oiliness in taste in $\{0,1,2,3,4\}$, where 0 means heavy or oily;
\item how frequently the user eats the SUSHI in  $\{0,1,2,3\}$, where 3 means frequently eat;
\item the normalized price;
\item how frequently the SUSHI is sold in sushi shop in $(0,1)$, where 1 means the most frequently.
\end{enumerate}

The agent features are
\begin{enumerate}
\item gender in  $\{0, 1\}$, where  0 means male and 1 means female;
\item age in range $\{0,1,2,3,4,5\}$, where 0 means 15-19; 1 means 20-29; 2 means 30-39; 3 means 40-49; 4 means 50-59; and 5 means 60 or elder;
\item the total time need to fill questionnaire form;
\item constant feature which is always 1.
\end{enumerate}
In figure~\ref{fig:sushi_acc}, every data point is the average of 26 independent experiments on 5-fold cross-validation, which means we run 130 times of training-test process. The confidence interval is calculated by the standard way of t-test.

\subsubsection{Definition of pairwise accuracy and pairwise MSE}\label{sec:accuracy}
Let $\vec\theta$ denote the parameters of $k$-$\plx$ and $R = [a_{i_1}\succ\cdots\succ a_{i_m}]$ denote a full ranking. The pairwise accuracy of $k$-$\plx$ on $R$ is the average of $k$-$\plx$'s prediction accuracy on all pairwise comparisons. Mathematically,
\begin{equation}\nonumber
\hat{\text{PA}}(R|\vec\theta) = \frac{1}{\binom{m}{2}}\sum_{\ell<\ell'} \mathbbm{1}\left({\Pr}_{k\text{-}\plx}\big(a_{i_{\ell}}\succ a_{i_{\ell'}} | \vec\theta\big) > 0.5 \right),
\end{equation}
where $\mathbbm{1}(\cdot)$ is the indicator function defined as follows
\begin{equation}\nonumber
\begin{split}
\mathbbm{1}(\kappa)&\triangleq\left\{\begin{array}{lcl}
1      &      & \text{if } \kappa \text{ is true}\\
0\;\;\;      &      & \text{otherwise}
\end{array}\right..\\
\end{split}
\end{equation}
For any profile $P = (R_1,\cdots,R_p)$, its pairwise accuracy is defined as the average accuracy of all rankings in it. Mathematically,
\begin{equation}\nonumber
\hat{\text{PA}}(P|\vec\theta) = \frac{1}{p}\sum_{j=1}^p \hat{\text{PA}}(R_j|\vec\theta).
\end{equation}
Similarly, using the same notations as above, pairwise MSE of ranking $R$ is defined as
\begin{equation}\nonumber
\hat{\text{PMSE}}(R|\vec\theta) = \frac{1}{\binom{m}{2}}\sum_{\ell<\ell'} \left(1-{\Pr}_{k\text{-}\plx}\big(a_{i_{\ell}}\succ a_{i_{\ell'}} | \vec\theta\big)\right)^2.
\end{equation}
For any profile $P = (R_1,\cdots,R_p)$, its pairwise MSE is defined as the average MSE of all rankings in it. Mathematically,
\begin{equation}\nonumber
\hat{\text{PMSE}}(P|\vec\theta) = \frac{1}{p}\sum_{j=1}^p \hat{\text{PMSE}}(R_j|\vec\theta).
\end{equation}

\subsubsection{Sushi dataset bechmarked by Pairwise MSE}\label{sec:mse}
Figure~\ref{fig:sushi_mse} plots the pairwise MSE of $k$-$\plxp$ with the same setting as Figure~\ref{fig:sushi_acc}. All our observations for pairwise accuracy in Section~\ref{sec:real_world} can also be observed in the plot for pairwise MSE. We note that the optimal pairwise MSE of $k$-$\plxp$ ($l=6$ for $k=1$ and $l=7$ for $k=2,3$)  becomes better when the number of clusters $k$ increases. 

\subsection{How likely is the full row rank condition violated?}~\label{app:rank}
We use sushi dataset and the same setting of agent/alternative feature as Section~\ref{sec:real_world}. We randomly sample $n$ agents from the sushi dataset without replacement. We define $p_{\text{FRR}}(n)$ as the probability that the full row rank condition is violated.  Figure~\ref{fig:rank} shows that $p_{\text{FRR}}(n)$ decays exponentially with $n$. Especially, when $n \geq 10$, $p_{\text{FRR}}(n) \leq 2.3\times 10^{-3}$, which means the full row rank condition is very unlikely violated when $n\geq 10$.

\begin{figure}[htp]
    \centering
    \includegraphics[width=0.46\textwidth]{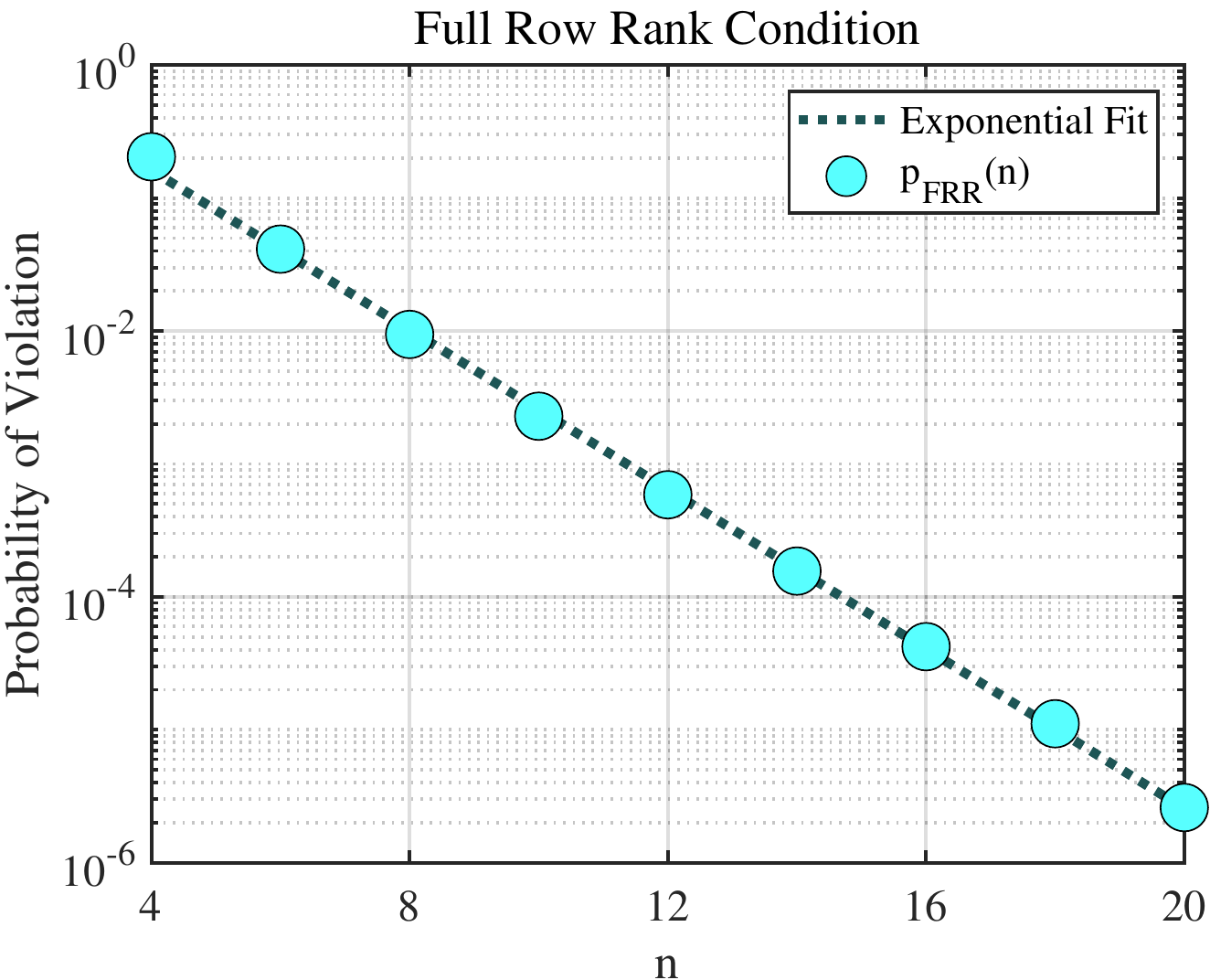}
    \caption{The experimental verification of the full row rank condition. Note that the vertical axis is in log-scale. We run $10^7$ independent trials for each data point.}\label{fig:rank}
\end{figure}

\subsection{Learning $\text{RUM}_{\mathcal{X}}$ from Synthetic Data}
We fix $m=10$ and $d=6$. For each agent and each alternative, the feature vector is generated in $[0,1]$ uniformly at random. Each component in $\vec\beta$ is also generated uniformly at random in $[0, 1]$. Each alternative is included in the $l$-way order with $p$ probability. 
All algorithms were implemented in MATLAB (with the built-in \texttt{fmincon} for MLE) and tested on a Windows 11 desktop with AMD 2700X CPUs each clocked at 4.0 GHz. Results are shown in Figure~\ref{fig:RUM_feature}. Values are computed by averaging over $5000$ trials.

\paragraph{Observations.} 
Figure~\ref{fig:RUM_feature} illustrates the comparison between MLE computed using uniform or weighted CLL. We observe that weighted outperforms uniform w.r.t.~both statistical efficiency and computational efficiency (except the computational efficiency for $p=0.2$). Therefore, adding weight might be a more favorable choice in practice. 

\section{Proofs}

\subsection{Useful Lemmas}\label{sec:lemmas}
We first show a lemma, which will be frequently used in the proofs of theorems in this paper.

\begin{lem}\label{lem:rank}
For any model $\plx$ where $\mx$ has $d$ rows, we have $d=\rank(\mx')$ if and only if there does not exist $\vec\beta^{(1)}\ne\vec\beta^{(2)}$, s.t. for any $j\in\{1, \ldots, n\}$ and any $i_1, i_2\in\{1, \ldots, m\}$ where $i_1\ne i_2$, $$(\vec\beta^{(1)}-\vec\beta^{(2)})(\vec x_{ji_1}-\vec x_{ji_2})=0.$$
\end{lem}
\begin{proof}
Because $\bignorm(\mx)$ has $d$ rows, it always holds that $d\ge\rank(\bignorm(\mx))$. Let $\vec 0$ denote a zero vector of appropriate dimension.

\noindent{\bf ``if" direction.} For the purpose of contradiction suppose $d>\rank(\bignorm(\mx))$. Then there exists vector $\vec\Delta\ne\vec 0$ s.t. $\bignorm(\mx)^\top\cdot\vec\Delta=\vec 0$. We construct $\vec\beta^{(1)}$ and $\vec\beta^{(2)}$ s.t. $\vec\beta^{(1)}-\vec\beta^{(2)}=\vec\Delta$. If $i_1=1$ or $i_2=1$, then either $\vec x_{ji_1}-\vec x_{ji_2}$ or $\vec x_{ji_2}-\vec x_{ji_1}$ is one column of $\bignorm(\mx)$. Therefore, we found $\beta^{(1)}$ and $\beta^{(2)}$ s.t. $(\vec\beta^{(1)}-\vec\beta^{(2)})(\vec x_{ji_1}-\vec x_{ji_2})=0$, which is a contradiction. If neither of $i_1$ and $i_2$ is $1$, then by taking out two columns in $\bignorm(\mx)$, we have
\begin{align*}
&(\vec\beta^{(1)}-\vec\beta^{(2)})(\vec x_{ji_1}-\vec x_{j1})=0\\
&(\vec\beta^{(1)}-\vec\beta^{(2)})(\vec x_{ji_2}-\vec x_{j1})=0.
\end{align*}
We get $(\vec\beta^{(1)}-\vec\beta^{(2)})(\vec x_{ji_1}-\vec x_{ji_2})=0$ by subtracting one from the other, which is a contradiction.

\noindent{\bf "only if" direction.} For the purpose of contradiction suppose there exist $\vec\beta^{(1)}, \vec\beta^{(2)}, j, i_1, i_2$ s.t. for any $j\in\{1, \ldots, n\}$ and any $i_1, i_2\in\{1, \ldots, m\}$ where $i_1\ne i_2$, $(\vec\beta^{(1)}-\vec\beta^{(2)})(\vec x_{ji_1}-\vec x_{ji_2})=0$. Then there exist a nonzero vector $\vec\beta^{(1)}-\vec\beta^{(2)}$ s.t. $\bignorm(\mx)^\top\cdot(\vec\beta^{(1)}-\vec\beta^{(2)})$, which implies that $\bignorm(\mx)$ is not full rank, i.e., $\rank(\bignorm(\mx))<d$, which is a contradiction.
\end{proof}


The next lemma will be used in the proof of the RMSE bound for MLE (Theorem~\ref{thm:msebound}).

\begin{lem}\label{lem:gradzero}
Given $\plxp$ with ground truth parameter $\vec\beta_0$. Let $L(\vec\beta)=\sum^n_{j=1}\ln\Pr\nolimits_{\plx}(R_j|\vec\beta)$, which is the objective function of MLE defined in~\eqref{eq:mle}. We have $E[\nabla L(\vec\beta_0)]=\vec 0$, with the expectation of each summand being zero, where the expectation is taken over orders generated from $\plxp$.
\end{lem}
\begin{proof}
We focus on an event of ranking any alternative $a_i$ at the top among any subset of alternatives $\ma'$ where $a_i\in\ma'$ since $L$ is constructed of such events. Let $l(\vec\beta)$ denote the likelihood of such an event. We have
$$
l(\vec\beta)=\vec\beta\cdot\vec x_{ji}-\ln\sum_{a_{i'}\in\ma'}\exp(\vec\beta\cdot\vec x_{ji'})
$$
and
$$
\nabla_r l(\vec\beta)=x_{ji, r}-\frac {\sum_{a_{i'}\in\ma'}x_{ji', r}\exp(\vec\beta\cdot\vec x_{ji'})} {\sum_{a_{i'}\in\ma'}\exp(\vec\beta\cdot\vec x_{ji'})}
$$
The probability of each alternative being ranked at the top is $\frac {\exp(\vec\beta_0\cdot\vec x_{ji})} {\sum_{a_{i'}\in\ma'}
\exp(\vec\beta_0\cdot\vec x_{ji'})}$. Therefore, for any $r\in\{1, \ldots, d\}$,
\begin{align*}
E[\nabla_r l(\vec\beta)]&=\frac {\sum_{a_i\in\ma'}x_{ji, r}\exp(\vec\beta_0\cdot\vec x_{ji})} {\sum_{a_{i'}\in\ma'}
\exp(\vec\beta_0\cdot\vec x_{ji'})}\\
&-\frac {\sum_{a_{i'}\in\ma'}x_{ji', r}\exp(\vec\beta\cdot\vec x_{ji'})} {\sum_{a_{i'}\in\ma'}\exp(\vec\beta\cdot\vec x_{ji'})}
\end{align*}
It's easy to see $E[\nabla_r l(\vec\beta_0)]=0$. Therefore $E[\nabla L(\vec\beta_0)]=\vec 0$.
\end{proof}

\subsection{Proof of Theorem~\ref{thm:idplx}}
\label{sec:proofthmidplx}
\appThm{thm:idplx}{
For any $\mx$,  
$\plxp$ is identifiable if and only if $\bignorm(\mx)$ has full row rank. }
\begin{proof}

We first show that the $\vec\phi$ part is always identifiable because different $\vec\phi$ parameters lead to different distributions over top-$l$ structures. Then we prove the theorem for the $\vec\beta$ part  
by analyzing the events of top choices over a subset of alternatives, observing that the probability of any top-$l$ order is the product of several probabilities of top choices over a subset of alternatives. 

Formally, the $\vec\phi$ parameter is always identifiable because for any different $\vec\phi$, the distribution of structures will be different, which contradicts the definition of identifiability. We only need to prove that $\vec\beta$ parameter is identifiable. 

{\bf ``if" direction.} It is not hard to see that if $\vec\beta$ is identifiable under the case where $\phi_{1}=1$ and $\phi_l=0$ for all $l\ge 2$ (the model generates top-$1$ orders only), $\vec\beta$ is identifiable for all any appropriate $\vec\phi$. So we focus on the $\phi_{1}=1$ case. For the purpose of contradiction suppose $\vec\beta$ parameter is not identifiable. There exist $\vec\beta^{(1)}\ne\vec\beta^{(2)}$ leading to the same distribution over top-$1$ orders. Then for any $j$ and any $i_1\ne i_2$, we have 
$
\frac{\exp({\vec\beta^{(1)}\cdot\vec x_{ji_1}})}{\sum^m_{i=1}\exp(\vec\beta^{(1)}\cdot\vec x_{ji})}
=\frac{\exp({\vec\beta^{(2)}\cdot\vec x_{ji_1}})}{\sum^m_{i=1}\exp(\vec\beta^{(2)}\cdot\vec x_{ji})}$ and $
\frac{\exp({\vec\beta^{(1)}\cdot\vec x_{ji_2}})}{\sum^m_{i=1}\exp(\vec\beta^{(1)}\cdot\vec x_{ji})}
=\frac{\exp({\vec\beta^{(2)}\cdot\vec x_{ji_2}})}{\sum^m_{i=1}\exp(\vec\beta^{(2)}\cdot\vec x_{ji})}$.
This simplifies to
\begin{equation}\label{eq:prop}
\frac{\exp({\vec\beta^{(1)}\cdot\vec x_{ji_1}})}{\exp(\vec\beta^{(1)}\cdot\vec x_{ji_2})}=\frac{\exp({\vec\beta^{(2)}\cdot\vec x_{ji_1}})}{\exp(\vec\beta^{(2)}\cdot\vec x_{ji_2})},
\end{equation}
which further simplifies to $\exp({\vec\beta^{(1)}(\vec x_{ji_1}-\vec x_{ji_2}})=\exp({\vec\beta^{(2)}(\vec x_{ji_1}-x_{ji_2})})$, and therefore $(\vec\beta^{(1)}-\vec\beta^{(2)})(\vec x_{ji_1}-\vec x_{ji_2})=0.$

By Lemma~\ref{lem:rank} (see Appendix~\ref{sec:lemmas}), $d>\text{rank}(\bignorm(\mx))$, which is a contradiction.

{\bf ``only if" direction.} For the purpose of contradiction suppose $d>\rank(\bignorm(\mx))$. By Lemma~\ref{lem:rank} (see Appendix~\ref{sec:lemmas}), there exists $\vec\beta^{(1)}\ne\vec\beta^{(2)}$, s.t. for any $j\in\{1, \ldots, n\}$ and any $i_1, i_2\in\{1, \ldots, m\}$ where $i_1\ne i_2$, $(\vec\beta^{(1)}-\vec\beta^{(2)})(\vec x_{ji_1}-\vec x_{ji_2})=0.$ This implies \eqref{eq:prop}. Now we focus on an event $E$, which is selecting an alternative $a_i$ from a subset of alternatives $\ma'$ where $a_i\in\ma'$. Due to \eqref{eq:prop}, for any alternative $a_{i'}\in\ma'$ and $a_{i'}\ne a_i$,

$$\frac{\exp({\vec\beta^{(1)}\cdot\vec x_{ji}})}{\exp(\vec\beta^{(1)}\cdot\vec x_{ji'})}=\frac{\exp({\vec\beta^{(2)}\cdot\vec x_{ji}})}{\exp(\vec\beta^{(2)}\cdot\vec x_{ji'})}$$

It's not hard to see
$$\frac{\exp({\vec\beta^{(1)}\cdot\vec x_{ji}})}{\sum_{a_{i'}\in\ma'}\exp(\vec\beta^{(1)}\cdot\vec x_{ji'})}=\frac{\exp({\vec\beta^{(2)}\cdot\vec x_{ji}})}{\sum_{a_{i'}\in\ma'}\exp(\vec\beta^{(2)}\cdot\vec x_{ji'})}$$
which indicates $\Pr\nolimits_{\plx}(E|\vec\beta^{(1)})=\Pr\nolimits_{\plx}(E|\vec\beta^{(2)})$, where $E$ can be any top-$1$ order over any subset of alternatives. Then it is easy to see that for any top-$l$ order $O$, we have $\Pr\nolimits_{\plxp}(O|\vec\beta^{(1)})=\Pr\nolimits_{\plxp}(O|\vec\beta^{(2)})$ by definition of $\plxp$. The model is not identifiable, which is a contradiction.
\end{proof}

\subsection{Proof of Corollary~\ref{coro:idbilinear}}
\label{sec:proof:thm:idbilinear}

\appCoro{coro:idbilinear}{
For any model $\plyzp$, where $Y$ is an $L$-by-$n$ matrix and $Z$ is a $K$-by-$m$ matrix, $\plyzp$ is identifiable if and only if both $Y$ and $\norm(Z)$ have full row rank.}
\begin{proof}
Since $\plyzp$ is a special case of $\plxp$ by letting $\mx=Y\otimes Z$ and $d=K\times L$ ($K$ and $L$ are the number of rows in $Y$ and $Z$, respectively.), it is not hard to see that $\bignorm(\mx)=Y\otimes\norm(Z)$. By Theorem~\ref{thm:idplx}, $\plx$ is identifiable if and only if $d=\rank(\bignorm(\mx))$. Therefore, $\plyzp$ is identifiable if and only if $\rank(Y\otimes \norm(Z))=K\times L$. Due to \citep[Theorem 4.2.15]{Roger94:Topics}, $\rank(Y\otimes \norm(Z))=\rank(Y)\rank(\norm(Z))$. Also due to the fact that $\rank(\norm(Z))\le K$ and $\rank(Y)\le L$, we have $\rank(Y\otimes \norm(Z))=K\times L$ if and only if $\rank(\norm(Z))=K$ and $\rank(Y)=L$.
\end{proof}

\subsection{Proof of Theorem~\ref{thm:idkplx}}
\label{sec:proofidmixture}

\appThm{thm:idkplx}{
If $k$-PL is identifiable, then for any $\mx$ such that $\bignorm(\mx)$ has full row rank,  any parameter $(\vec\alpha, \vec\beta, \vec\phi)$ with $\phi_{m-1}>0$ is identifiable in $k$-$\plxp$.}
\begin{proof} The
$\vec\phi$ parameter is always identifiable because a different $\vec\phi'$ will lead to a different distribution over structures of partial orders. We only need to prove that the remaining parts of the parameter ($\vec\alpha$, $\vec\beta^{(r)}$ for each $r\in\{1, \ldots, k\}$) are identifiable.

Let $\vec\gamma=(\vec\phi, \vec\alpha, \vec\beta^{(1)}, \ldots, \vec\beta^{(k)})$. For the purpose of contradiction suppose there exists another parameter $\vec\gamma'=(\vec\phi, \vec\alpha', \vec\beta'^{(1)}, \ldots, \vec\beta'^{(k)})$ s.t. $\vec\gamma\ne\vec\gamma'$ and $\vec\gamma$ and $\vec\gamma'$ lead to the same distribution over rankings for each agent. For each $j\in\{1, \ldots, n\}$, each $t\in\{1, \ldots, k\}$, and each $i\in\{1, \ldots, m\}$ we define $\theta^{(j, t)}_i=\vec\beta^{(t)}\cdot\vec x_{ji}$ and $\theta'^{(j, t)}_i=\vec\beta'^{(t)}\cdot\vec x_{ji}$, which can be viewed as the standard PL parameters for agent $j$ and alternative $a_i$ for $t$-th component of $k$-$\plxp$. For convenience we define $\vec\theta^{(j, t)}=(\theta^{(j, t)}_1, \ldots, \theta^{(j, t)}_m)$ and $\vec\theta'^{(j, t)}=(\theta'^{(j, t)}_1, \ldots, \theta'^{(j, t)}_m)$.

We now claim that the mixing coefficient parts of $\vec\gamma$ and $\vec\gamma'$ are equal, i.e., $\vec\alpha=\vec\alpha'$ modulo label switching. For the purpose of contradiction suppose the mixing coefficients parts are different, then for each agent $j$, there exist two $k$-PL parameters that lead to the same distribution over rankings. This contradicts the condition that $k$-PL is identifiable.

Still due to identifiability of $k$-PL, given any agent $j\in\{1, \ldots, n\}$ and any component $t\in\{1, \ldots, k\}$, $\vec\theta^{(j, t)}$ and $\vec\theta'^{(j, t)}$ must also be the same modulo parameter shifting, which means that for any $i\in\{2, \ldots, m\}$, $\theta^{(j, t)}_i-\theta^{(j, t)}_1=\theta'^{(j, t)}_i-\theta'^{(j, t)}_1$. This implies that $\norm(X_j)^
\top\vec\beta^{(t)}=\norm(X_j)^\top\vec\beta'^{(t)}$. Since this holds for all $j\in\{1, \ldots, n\}$, we have $\bignorm(\mx)^
\top(\vec\beta^{(t)}-\vec\beta'^{(t)})=\vec 0$. Since $\bignorm(\mx)$ has full row rank, we have $\vec\beta^{(t)}=\vec\beta'^{(t)}$, which holds for all $t\in\{1, \ldots, k\}$. Therefore, we have $\vec\gamma=\vec\gamma'$, which is a contradiction.
\end{proof}

\subsection{Proof of Theorem~\ref{thm:log-concavity}}
\label{sec:proofthmconcavity}
\appThm{thm:log-concavity}{
For any $\plxp$ and any data $P=(O_1, \ldots, O_n)$, the log likelihood function in \eqref{eq:mle} is strictly concave if and only if $\bignorm(\mx)$ has full row rank.}
\begin{proof}
This proof consists of two parts: (I)  for any $j\in\{1, \ldots, n\}$, $\ln\Pr_{\plx}(O_j|\vec\beta)$ is concave, and (II) there exists $j\in\{1, \ldots, n\}$ s.t. $\ln\Pr_{\plx}(O_j|\vec\beta)$ is strictly concave if and only if $\bignorm(\mx)$ has full rank. 

\noindent{\bf Part I.} Due to Definition~\ref{def:fpl}, it is sufficient to prove that for any subset of alternatives $\ma'\subset\ma$ and any $a_i\in\ma'$, $\ln\frac {\exp(\vec\beta\cdot\xji)} {\sum_{a_{i'}\in\ma'}\exp(\vec\beta\cdot\xjp{i'})}$ is concave.

Let $f(\vec\beta)=\ln\frac {\exp(\vec\beta\cdot\xji)} {\sum_{a_{i'}\in\ma'}\exp(\vec\beta\cdot\xjp{i'})}=\vec\beta\cdot\xji-\ln\sum_{a_{i'}\in\ma'}\exp(\vec\beta\cdot\xjp{i'})$. The first term $\vec\beta\cdot\xji$ is concave (linear). It is sufficient to prove $-\ln\sum_{a_{i'}\in\ma'}\exp(\vec\beta\cdot\xjp{i'})$ is also concave. Let $g(\vec\beta)=\ln\sum_{a_{i'}\in\ma'}\exp(\vec\beta\cdot\xjp{i'})$. Our goal is to prove that for any $\vec\beta^{(1)}\ne\vec\beta^{(2)}$, 
$\frac 1 2(g(\vec\beta^{(1)})+g(\vec\beta^{(2)}))>g(\frac {\vec\beta^{(1)}+\vec\beta^{(2)}} 2)$. We compute the difference between the left hand side and the right hand side as follows.
\begin{align*}
&g(\vec\beta^{(1)})+g(\vec\beta^{(2)})-2g(\frac {\vec\beta^{(1)}+\vec\beta^{(2)}} 2)\\
=&\ln\sum_{a_{i'}\in\ma'}\exp(\vec\beta^{(1)}\cdot\xjp{i'})+\ln\sum_{a_{i'}\in\ma'}\exp(\vec\beta^{(2)}\cdot\xjp{i'})\\
&-2\ln\sum_{a_{i'}\in\ma'}\exp(\frac {\vec\beta^{(1)}+\vec\beta^{(2)}} 2\cdot\xjp{i'})\\
=&\ln((\sum_{a_{i'}\in\ma'}\exp(\vec\beta^{(1)}\cdot\xjp{i'}))(\sum_{a_{i'}\in\ma'}\exp(\vec\beta^{(2)}\cdot\xjp{i'})))\\
&-\ln(\sum_{a_{i'}\in\ma'}\exp(\frac {\vec\beta^{(1)}+\vec\beta^{(2)}} 2\cdot\xjp{i'}))^2
\end{align*}
Due to monotonicity of $\ln$ function, we only need to show that 
\begin{align*}
&(\sum_{a_{i'}\in\ma'}\exp(\vec\beta^{(1)}\cdot\xjp{i'}))(\sum_{a_{i'}\in\ma'}\exp(\vec\beta^{(2)}\cdot\xjp{i'}))\\
-&(\sum_{a_{i'}\in\ma'}\exp(\frac {\vec\beta^{(1)}+\vec\beta^{(2)}} 2\cdot\xjp{i'}))^2>0.
\end{align*}

We will show that the left-hand-side can be written as a sum of squares, where at least one of them is positive. 
We have
\begin{align}
&(\sum_{a_{i'}\in\ma'}\exp(\vec\beta^{(1)}\cdot\xjp{i'}))(\sum_{a_{i'}\in\ma'}\exp(\vec\beta^{(2)}\cdot\xjp{i'}))\notag\\
&-(\sum_{a_{i'}\in\ma'}\exp(\frac {\vec\beta^{(1)}+\vec\beta^{(2)}} 2\cdot\xjp{i'}))^2\notag\\
=&\sum_{a_{i'}\in\ma'}\exp((\vec\beta^{(1)}+\vec\beta^{(2)})\cdot\xjp{i'})\notag\\
&+\sum_{a_{i'_1}, a_{i'_2}\in\ma',i'_1\ne i'_2}\exp(\vec\beta^{(1)}\cdot\xjp{i'_1}+\vec\beta^{(2)}\cdot\xjp{i'_2}))\notag\\
&-(\sum_{a_{i'}\in\ma'}\exp((\vec\beta^{(1)}+\vec\beta^{(2)})\cdot\xjp{i'})\notag\\
&+\sum_{a_{i'_1}, a_{i'_2}\in\ma',i'_1\ne i'_2}\exp(\frac {\vec\beta^{(1)}+\vec\beta^{(2)}} 2\cdot(\xjp{i'_1}+\xjp{i'_2})))\notag\\
=&\sum_{a_{i'_1}, a_{i'_2}\in\ma',i'_1<i'_2}(\exp(\vec\beta^{(1)}\cdot\xjp{i'_1}+\vec\beta^{(2)}\cdot\xjp{i'_2})\notag\\
&+\exp(\vec\beta^{(1)}\cdot\xjp{i'_2}+\vec\beta^{(2)}\cdot\xjp{i'_1})\notag\\
&-2\exp(\frac {\vec\beta^{(1)}+\vec\beta^{(2)}} 2\cdot(\xjp{i'_1}+\xjp{i'_2})))\notag\\
=&\sum_{a_{i'_1}, a_{i'_2}\in\ma',i'_1<i'_2}(\exp(\frac {\vec\beta^{(1)}\cdot\xjp{i'_1}+\vec\beta^{(2)}\cdot\xjp{i'_2}} 2)\notag\\
&-\exp(\frac {\vec\beta^{(1)}\cdot\xjp{i'_2}+\vec\beta^{(2)}\cdot\xjp{i'_1}} 2))^2\label{eq:conc}
\end{align}
Not we have proved concavity. 

\noindent{\bf Part II.} ``if direction": we need to prove that there exists $j\in\{1, \ldots, n\}$ and $i'_1<i'_2, i'_1, i'_2\in\ma'$ s.t. \eqref{eq:conc} is nonzero. For the purpose of contradiction suppose for any $j=1, \ldots, n$, any distinct $i'_1, i'_2\in\ma$,
\begin{equation*}
\vec\beta^{(1)}\cdot\xjp{i'_1}+\vec\beta^{(2)}\cdot\xjp{i'_2}-\vec\beta^{(1)}\cdot\xjp{i'_2}-\vec\beta^{(2)}\cdot\xjp{i'_1}=0.
\end{equation*}
This simplifies to
\begin{equation*}
(\vec\beta^{(1)}-\vec\beta^{(2)})(\xjp{i'_1}-\xjp{i'_2})=0
\end{equation*}
By Lemma~\ref{lem:rank}, this does not hold for every $j$ and distinct $i'_1, i'_2$ when $d=$rank$(\bignorm(\mx))$, which is a contradiction.

``only if" direction. We prove that if $\bignorm{\mx}$ does not have full rank, the objective function is not strictly concave. Since $\bignorm{\mx}$ does not have full rank, we can find $\vec\beta^{(1)}$ and $\vec\beta^{(2)}$ s.t. $(\vec\beta^{(1)}-\vec\beta^{(2)})(\xjp{i'_1}-\xjp{i'_2})=0$ holds for every agent $j$ and every $i'_1, i'_2$ pair  (Lemma~\ref{lem:rank}). This means there exist $\vec\beta^{(1)}$ and $\vec\beta^{(2)}$ s.t. $\frac 1 2(g(\vec\beta^{(1)})+g(\vec\beta^{(2)}))=g(\frac {\vec\beta^{(1)}+\vec\beta^{(2)}} 2)$. The objective function is not strictly concave.
\end{proof}

\subsection{Proof of Lemma~\ref{lem:bound}}

\appLem{lem:bound}{
For any $\plxp$ and data $P$, if Assumption~\ref{ass:bound} holds then the MLE in~\eqref{eq:mle} is bounded.}
\label{sec:proof:lem:bound}
\begin{proof}
We only need to prove that when any entry of $\vec\beta$ goes to infinity, the objective function in \eqref{eq:mle} $LL(P|\vec\beta)=\sum^n_{j=1}\ln\Pr\nolimits_{\plx}(O_j|\vec\beta)$ goes to negative infinity, i.e., the probability of some ranking in $P$ approaches zero.

In the rest of the proof, we focus on the probability of an event $E_j$, which is $a_{i_1}$ being ranked at the top among a subset of alternatives $\ma'$ which contains both $a_{i_1}$ and $a_{i_2}$ by agent $j$, under $\plx$. We have
\begin{align*}
\Pr\nolimits_{\plx}(E_j|\vec\beta)=&\frac {\exp(\vec\beta\cdot\vec x_{ji_1})} {\exp(\vec\beta\cdot\vec x_{ji_2})+\sum_{a_{i'}\in\ma', i'\ne i_2}\exp(\vec\beta\cdot\vec x_{ji'})}\\
=&\frac {1} {\exp(\vec\beta\cdot(\vec x_{ji_2}-\vec x_{ji_1}))+M},
\end{align*}
where $M$ is positive. We will show that for any $r\in\{1, \ldots, d\}$, under all cases of Assumption~\ref{ass:bound}, we can find an event $E$ s.t. as $\beta_r\rightarrow\pm\infty$, $\Pr\nolimits_{\plx}(E_j|\vec\beta)\rightarrow 0$.

\noindent{\bf Case 1:} for any $r\in\{1, \ldots, d\}$, there exist $j_1, j_2, i_1, i_2$ s.t. $\xi_{j_1, i_1i_2}\xi_{j_2, i_1i_2}>0$, $(x_{j_1i_1, r}-x_{j_1i_2, r})(x_{j_2i_1, r}-x_{j_2i_2, r})<0$.

We only need to consider the case where $\xi_{j_1, i_1i_2}>0$, $\xi_{j_2, i_1i_2}>0$. The other case ($\xi_{j_1, i_1i_2}<0$, $\xi_{j_2, i_1i_2}<0$) can be converted to this case by switching the roles of $i_1$ and $i_2$. Similarly, we only need to consider the case where $x_{j_1i_1, r}-x_{j_1i_2, r}>0$, $x_{j_2i_1, r}-x_{j_2i_2, r}<0$. The other case can be converted to this case by switching the roles of $j_1$ and $j_2$.

In this case, we let $E_j$ be the event of ranking $a_{i_1}$ at the top among a subset of alternatives containing both $a_{i_1}$ and $a_{i_2}$. For any $r\in\{1, \ldots, d\}$, if $\beta_r\rightarrow-\infty$, we have $\exp(\beta_r(x_{j_1i_2, r}-x_{j_1i_1, r}))\rightarrow+\infty$, which means $\Pr(E_{j_1}|\vec\beta)\rightarrow 0$; if $\beta_r\rightarrow+\infty$, we have $\exp(\beta_r(x_{j_2i_2, r}-x_{j_2i_1, r}))\rightarrow+\infty$, which means $\Pr(E_{j_2}|\vec\beta)\rightarrow 0$.

\noindent{\bf Case 2:} for any $r\in\{1, \ldots, d\}$, there exist $j_1, j_2, i_1, i_2$ s.t. $\xi_{j_1, i_1i_2}\xi_{j_2, i_1i_2}<0$, $(x_{j_1i_1, r}-x_{j_1i_2, r})(x_{j_2i_1, r}-x_{j_2i_2, r})>0$.

Again, we only need to consider the case where $\xi_{j_1, i_1i_2}>0$, $\xi_{j_2, i_1i_2}<0$, $x_{j_1i_1, r}-x_{j_1i_2, r}>0$, $x_{j_2i_1, r}-x_{j_2i_2, r}>0$. All the other case can be converted to this case by switching the roles of $j_1$ and $j_2$, $i_1$ and $i_2$, or both.

For any $r\in\{1, \ldots, d\}$, if $\beta_r\rightarrow-\infty$, let $E_j$ be the event of ranking $a_{i_1}$ among a subset of alternatives containing $a_{i_1}$ and $a_{i_2}$. Then we have $\exp(\beta_r(x_{j_1i_2, r}-x_{j_1i_1, r}))\rightarrow+\infty$, which means $\Pr(E_{j_1}|\vec\beta)\rightarrow 0$; if $\beta_r\rightarrow+\infty$, let $E_j$ be the event of ranking $a_{i_2}$ among a subset of alternatives containing $a_{i_1}$ and $a_{i_2}$.
we have $\exp(\beta_r(x_{j_2i_1, r}-x_{j_2i_2, r}))\rightarrow+\infty$, which means $\Pr(E_{j_2}|\vec\beta)\rightarrow 0$.
\end{proof}

\subsection{Proof of Theorem~\ref{thm:msebound}}
\label{sec:msebound}
\appThm{thm:msebound}{
Given any $\plxp$ over $m$ alternatives and $n$ agents with the feature matrix $\mx\in\mathbb R^{d\times mn}$. Define $L(\vec\beta)=\frac 1 n \sum^n_{j=1}\ln\Pr\nolimits_{\plx}(O_j|\vec\beta)$, which is $\frac 1 n$ of the objective function in \eqref{eq:mle}. Let $H(\vec\beta)$ denote the Hessian matrix of $L(\vec\beta)$ and $\lambda_1(\vec\beta)$ be the smallest eigenvalue of $-H(\vec\beta)$. Let $\vec\beta_0$ denote the ground truth parameter and $\vec\beta^*$ denote the estimated parameter that is computed using \eqref{eq:mle}. Define $\lambda_{\min}=\min_{0\le\sigma\le 1}\lambda_1(\sigma\vec\beta^*+(1-\sigma)\vec\beta^0)$.

If \bignorm($\mx$) has full row rank and Assumption~\ref{ass:bound} holds, then for any $0<\delta<1$, with probability $1-\delta$, 
\begin{equation}\label{eq:samplecomplexity1}
||\vec\beta^*-\vec\beta_0||_2\le\frac {\sqrt{8(m-1)^2c^2d\ln(\frac {2d} {\delta})}} {\lambda_{\min}\sqrt{n}},
\end{equation}
where $c$ is the difference between the largest and the smallest entries of $\mx$.}
\begin{proof}
It is easy to see $\vec\beta^*$ defined in $\eqref{eq:mle}$ also maximizes $L(\vec\beta)$.

Let the order $a_{i_{j, 1}}\succ a_{i_{j, 2}}\succ\ldots\succ a_{i_{j, l_j}}\succ\text{others}$ denote the order $O_j$, i.e., $a_{i_{j, p}}$ is the alternative that is ranked at position $p$ by agent $j$. Accordingly, the feature vector for $a_{i_{j, p}}$ is denoted by $\vec x_{ji_{j, p}}$. Then $L(\vec\beta)$ can be written explicitly as
$
L(\vec\beta)=\frac 1 n\sum^n_{j=1}\sum^{l_j}_{p=1}(\vec\beta\cdot\vec x_{ji_{j, p}}-\ln\sum^m_{q=p}\exp(\vec\beta\cdot\vec x_{ji_{j, q}}))
$.

Define $\Delta=\vec\beta^*-\vec\beta_0$ and let $\nabla L(\vec\beta)$ denote the gradient of $L(\vec\beta)$. We have
\begin{align}\label{eq:gradbound}
&L(\vec\beta^*)-L(\vec\beta_0)-\nabla^\top L(\vec\beta^0)\Delta\notag\\
&\ge -\nabla^\top L(\vec\beta^0)\Delta\ge -||\nabla L(\vec\beta^0)||_2||\Delta||_2
\end{align}
where the first inequality is because $L(\vec\beta^*)$ maximize $L(\vec\beta)$ and the second inequality is due to the Cauchy-Schwartz inequality. Let $H(\vec\beta)$ denote the Hessian matrix of $L(\vec\beta)$. Then by the mean value theorem, there exist a $\vec\beta'=\sigma\vec\beta^*+(1-\sigma)\vec\beta_0$ for some $0\le\sigma\le 1$ such that
\begin{align}\label{eq:meanvaluebound}
&L(\vec\beta^*)-L(\vec\beta_0)-\nabla^\top L(\vec\beta^*)\Delta=\frac 1 2\Delta^\top H(\vec\beta')\Delta\notag\\
&\le -\frac 1 2 \lambda_1(-H(\vec\beta'))||\Delta||^2_2\le -\frac 1 2 \lambda_{\min}||\Delta||^2_2
\end{align}
where $\lambda_1(-H(\vec\beta'))$ is the smallest eigenvalue of $-H(\vec\beta')$. Due to Theorem~\ref{thm:log-concavity}, $H(\vec\beta)$ is negative definite because $L(\vec\beta)$ is strictly concave when $d=\rank(\bignorm(\mx))$. Therefore $\lambda_1(-H(\vec\beta'))>0$. Combining \eqref{eq:gradbound} and \eqref{eq:meanvaluebound}, we have
\begin{equation}\label{eq:msebound1}
||\Delta||_2\le\frac {2||\nabla L(\vec\beta_0)||_2} {\lambda_{\min}}
\end{equation}

Now we bound $||\nabla L(\vec\beta_0)||_2$. The $r$-th entry of gradient is
$
\nabla_r L(\vec\beta)
=\frac 1 n\sum^n_{j=1}\sum^{m-1}_{p=1}(x_{ji_{j, p}, r}-\frac {\sum^m_{q=p}x_{ji_{j, q}, r}\exp(\vec\beta\cdot\vec x_{ji_{j, q}})} {\sum^m_{q=p}\exp(\vec\beta\cdot\vec x_{ji_{j, q}})})
$. It is shown by Lemma~\ref{lem:gradzero} that $E[\nabla_r L(\vec\beta_0)]=0$ with each summand expected to be zero. It's not hard to see that each summand $|x_{ji_{j, p}, r}-\frac {\sum^m_{q=p}x_{ji_{j, q}, r}\exp(\vec\beta\cdot\vec x_{ji_{j, q}})} {\sum^m_{q=p}\exp(\vec\beta\cdot\vec x_{ji_{j, q}})}|\le c$, where $c$ is the difference between the largest and the smallest entries of $\mx$. $\nabla_r L(\vec\beta)$ can be viewed as the mean of $n$ random variables bounded in $[-(m-1)c, (m-1)c]$. By Hoeffding's inequality, for any $r\in\{1, \ldots, d\}$ and any $\epsilon>0$,
\begin{equation*}
\Pr(|\nabla_r L_\beta(\vec\beta_0)|\ge \frac {\epsilon} {\sqrt{d}})\le 2\exp(-\frac {n\epsilon^2} {2(m-1)^2c^2d}).
\end{equation*}
Therefore, we have
\begin{equation*}
\Pr(||\nabla L_\beta(\vec\beta_0)||_2\le \epsilon)\ge 1 - 2d\exp(-\frac {n\epsilon^2} {2(m-1)^2c^2d}).
\end{equation*}
This inequality is obtained by applying union bound. Then from \eqref{eq:msebound1}, we have
$$\Pr(||\Delta||_2\le \frac {2\epsilon} {\lambda_{\min}})\ge 1 - 2d\exp(-\frac {n\epsilon^2} {2(m-1)^2c^2d}).$$ Then \eqref{eq:samplecomplexity} is obtained by letting $\delta=2d\exp(-\frac {n\epsilon^2} {2(m-1)^2c^2d})$
\end{proof}

\onecolumn

\section{Code for Key Algorithms}

\subsection{MLE for $\plxp$}

The objective function:
\begin{verbatim}
% obj_mle_plx.m
function nll = obj_mle_plx(beta, features, rankings, l)
% To compute the negative log likelihood
% beta: the parameter to be optimized
% features: m-by-d-by-n feature tensor
% rankings: full rankings from agents
% l: the number of top ranked alternatives to be used
% nll: the negative log likelihood
[n, ~] = size(rankings);
nll = 0;
for j = 1:n
    ranking = rankings(j, :);
    theta = features(:, :, j)*beta';
    gamma = exp(theta);
    s = sum(gamma);
    for i = 1:l
        nll = nll - theta(ranking(i)) + log(s);
        s = s - gamma(ranking(i));
    end
end
end
\end{verbatim}

\noindent MLE for $\plx$:
\begin{verbatim}
% mle_plx.m
function [beta, itr] = mle_plx(features, rankings, l, beta0)
% To compute the estimated beta given data
% features: m-by-d-by-n feature tensor
% rankings: full rankings from agents
% l: the number of top ranked alternatives to be used
% beta0: initial parameter (randomly chosen, any seed is fine)
% beta: optimal parameter given the data
% itr: iterations used to find beta
[beta, ~, ~, output] = fminunc(@(x) obj_mle_plx(x, features, rankings, l), beta0);
itr = output.iterations;
end
\end{verbatim}
\subsection{EM for $k$-$\plxp$}

The EM algorithm:
\begin{verbatim}
% em_kplx.m
function [alphas, betas, te, tm, estfull] = em_kplx( features, rankings, alphas0,
betas0, itr )
% Implementation of EM algorithm for k-PL-x
% features: m-by-d-by-n feature tensor
% rankings: rankings from agents
% alphas0: initial values for mixing coefficients (randomly chosen)
% betas0: initial values for parameters of each component model (randomly chosen)
% itr: number of iterations to run
% alphas: estimated mixing coefficients
% betas: estimated parameters for each component model
% te: E-step running time
% tm: M-step running time
% estfull: detailed estimates by iteration
[m, d, n] = size(features);
k = length(alphas0); % number of components
if k == 1
    itr = 1;
end
alphas = alphas0;
betas = betas0;
te = 0;
tm = 0;
estfull = zeros(itr, k*(d+1)+2);
ws = zeros(n, k);
for EMiter = 1:itr
    % E-step
    tes = tic;
    newalphas = zeros(1, k);
    for j = 1:n
        ranking = rankings(j, :);
        w = zeros(1, k);
        for r = 1:k
            w(r) = alphas(r)*prplx(features(:, :, j), ranking, betas(r,:));
        end
        w = w/sum(w);
        ws(j, :) = w;
        newalphas = newalphas + w;
    end
    tee = toc(tes);
    te = te + tee;
    % M-step
    tms = tic;
    alphas = newalphas / sum(newalphas);
    cps = zeros(k, d);
    for r = 1:k
        cps(r, :) = wmle_plx(features, rankings, ws(:, r), betas(r, :));
    end
    tme = toc(tms);
    tm = tm + tme;
    betas = cps;
    estfull(EMiter,:) = [alphas,reshape(betas',1,k*d),te,tm];
end
end
\end{verbatim}

\noindent Function \texttt{prplx}:
\begin{verbatim}
% prplx.m
function pr = prplx( feature, ranking, beta, k )
l = length(ranking);
% To compute the probability of a ranking under PL-x
% feature: feature matrix for the agent
% ranking: the agent's preference
% beta: the parameter of the model
% k: the number of top ranked alternatives to be used
% pr: the probability of this ranking
if nargin < 4
    k = l - 1;
end
theta = feature*beta';
gamma = exp(theta);
gamma = gamma/sum(gamma);
pr = 1;
s = 0;
for j = 1:l
    s = s + gamma(ranking(j));
end
for i = 1:k
    tt = gamma(ranking(i));
    pr = pr*tt/s;
    s = s - tt;
end
end
\end{verbatim}

\noindent Function \texttt{wmle\_plx}:
\begin{verbatim}
% wmle_plx.m
function beta = wmle_plx(features, rankings, ws, beta0)
% MLE algorithm for PL-x with different weights across rankings. 
% This algorithm is used in the M step of EM algorithm for k-PLx
% features: m-by-d-by-n feature tensor
% rankings: preferences from agents
% ws: weights of each ranking
% beta0: random initial value of the parameter
% beta: estimated parameter
beta = fminunc(@(x) obj_wmle_plx(x, features, rankings, ws), beta0);
end
\end{verbatim}

\noindent Function \texttt{obj\_wmle\_plx}
\begin{verbatim}
% obj_wmle_plx.m
function nll = obj_wmle_plx(beta, features, rankings, ws)
% Objective function of MLE algorithm for PLx with different weights on each ranking
% beta: parameter of the model
% features: m-by-d-by-n feature tensor
% rankings: preferences of agents
% ws: weights of the rankings
% nll: negative log likelihood
[n, ~] = size(rankings);
nll = 0;
for j = 1:n
    theta = features(:, :, j)*beta';
    gamma = exp(theta);
    gamma = gamma/sum(gamma);
    nll = nll - ws(j)*log(prpl(gamma, rankings(j,:)));
end
end
\end{verbatim}

\noindent Function \texttt{prpl}
\begin{verbatim}
% prpl.m
function pr = prpl( theta, ranking, k )
% to compute the probability of a ranking given PL
% theta: the parameter of PL
% ranking: the probability of which to be computed
% k: the number of top ranked alternatives considered
% pr: probability of the ranking
l = length(ranking);
if nargin < 3
    k = l - 1;
end
pr = 1;
s = 0;
for j = 1:l
    s = s + theta(ranking(j));
end
for i = 1:k
    tt = theta(ranking(i));
    pr = pr*tt/s;
    s = s - tt;
end
end
\end{verbatim}

\end{document}